 \documentclass[twocolumn,10pt]{jmlr} % W&CP article

\usepackage{booktabs}
\usepackage{siunitx}

\newcommand{\equal}[1]{{\hypersetup{linkcolor=black}\thanks{#1}}}
\usepackage{algorithm}
\usepackage[noend]{algorithmic}
\usepackage{enumitem}
\usepackage{multirow}
\usepackage{amsmath}
\usepackage[font=normalsize]{caption}
\usepackage{bm, bbm}
\usepackage{longtable}% for long tables
 \usepackage{booktabs}
 \usepackage{enumitem}
\setlist{itemsep=0.2ex, topsep=1ex, partopsep=0ex, parsep=0.5ex, leftmargin=2.3ex}
\usepackage{xtab,afterpage}

\theorembodyfont{\upshape}
\theoremheaderfont{\scshape}
\theorempostheader{:}
\theoremsep{\newline}

\def\concat{\mathbin\Vert}

 \title[Heart Failure Prediction Through Explainable AI]{Interpretable Survival Analysis 
for Heart Failure Risk Prediction}

\author{
       \Name{Mike Van Ness}\equal{Equal contribution.}
       \Email{mvanness@stanford.edu}\\ 
       \addr Stanford University
       \AND
       \Name{Tomas Bosschieter}\footnotemark[1]
       \Email{tomasbos@stanford.edu}\\ 
       \addr Stanford University
       \AND
      \Name{Natasha Din, MD}
       \Email{nd1n@stanford.edu}\\ 
       \addr VA Palo Alto Health Care System
       \AND
       \Name{Andrew Ambrosy, MD}
       \Email{Andrew.P.Ambrosy@kp.org}\\ 
       \addr Kaiser Permanente Northern California Division of Research
       \AND
       \Name{Alexander Sandhu, MD} % 
       \Email{ats114@stanford.edu}\\ 
       \addr Stanford Medicine
       \AND
       \Name{Madeleine Udell} 
       \Email{udell@stanford.edu}\\ 
       \addr Stanford University\\
       } 

\begin{document}

\maketitle

\begin{abstract}
  Survival analysis, or time-to-event analysis, is an important and widespread problem in healthcare research. 
  Medical research has traditionally relied on Cox models for survival analysis, due to their simplicity and interpretability.
  Cox models assume a log-linear hazard function as well as proportional hazards over time,
  and can perform poorly when these assumptions fail.
  Newer survival models based on machine learning avoid these assumptions and 
  offer improved accuracy, yet sometimes at the expense of model interpretability, 
  which is vital for clinical use.
  We propose a novel survival analysis pipeline that is both interpretable and competitive with state-of-the-art survival models.
  Specifically, we use an improved version of survival stacking to transform a survival analysis problem to a classification problem, ControlBurn to perform feature selection, and Explainable Boosting Machines to generate interpretable predictions.
  To evaluate our pipeline, we predict risk of heart failure using a large-scale EHR database. Our pipeline achieves state-of-the-art performance and provides interesting and novel insights about risk factors for heart failure. 
\end{abstract}
\begin{keywords}
explainability, healthcare, heart failure, survival analysis, generalized additive models
\end{keywords}

\section{Introduction} \label{sec:introduction}
% XX is an important problem in machine learning and healthcare.  (Make
% sure that the clinicians can see the relevance! \emph{Unclear clinical
%   relevance is a major reason that otherwise strong-looking papers are
%   scored low/rejected.})

% Addressing this problem is challenging because XX.  (Make sure that
% you connect to the machine learning here.)  

% Others have tried, but XX remains tough.  (Acknowledge related work.)

% In this work, we...

% As you write, keep in mind that MLHC papers are meant to be read by
% computer scientists and clinicians.  In the later sections, you might
% have to use some medical terminology that a computer scientist may not
% be familiar with, and you might have to use some math that a clinician
% might not be familiar with.  That's okay, as long as you've done your
% best to make sure that the core ideas can be understood by an informed
% reader in either community.

Predicting individualized risk of developing a disease or condition, e.g.~heart failure, is a classical and important problem in medical research. While this risk modeling is sometimes accomplished using classification models, healthcare data often contains many right-censored samples: patients who are lost to followup before the end of the prediction window. 
Classification models cannot directly handle such censored data; 
instead, these samples are often discarded, losing valuable signal.
Moreover, the classification approach requires fixing a specific risk window 
(e.g., whether a patient has developed heart failure in the first five years after their initial visit),
and cannot exploit the time-to-event signal.

Survival analysis, in contrast, handles right-censoring by modeling the time until a patient develops a condition as a continuous random variable $T$ in order to estimate the survival curve $P(T > t)$. 
Survival analysis tasks have typically been modeled using Cox proportional hazards models  \citep{cox1972regression}, which assume that the log of the hazard function is a linear function of patient covariates plus a time-dependent intercept, resulting in proportional hazards over time.
Cox models account for right-censored data by optimizing the partial likelihood 
\begin{equation}\label{eq:Cox-partial-likelihood}
    L(\beta) = \prod_{i=1}^{K} \frac{\exp(x_i^\top \beta)}{\sum_{j \in R(t_i)} \exp(x_j^\top \beta)},
\end{equation}
where $K$ denotes the number of uncensored event times $t_1 \leq t_2 \leq \cdots \leq t_K$ and $R(t_i) = \{j : t_j \geq t_i\}$ represents the \textit{risk set} of patients that have not passed away yet at time $t_i$. Since the sum in the denominator of Eq.~\eqref{eq:Cox-partial-likelihood} is over \emph{all} patients in $R(t_i)$ regardless of censoring, Cox models can learn signal from both censored and uncensored patients. Cox models have become the standard in survival analysis due to their simplicity, interpretability, and natural handling of time-to-event data with censoring.

Naturally, the machine learning community has proposed many models that improve risk prediction accuracy past Cox models.
For example, the package \texttt{scikit-survival} \citep{polsterl2020scikit} contains several machine learning models for survival analysis, most of which do not assume a log-linear hazard function. Further, some recent machine learning approaches deal with right-censored data by adopting pseudovalues or pseudo-observations \citep{andersen2010pseudo}, which requires that an estimator for the survival curve on the complete data is available, e.g.~through the Kaplan-Meier estimator. Pseudovalues have enabled researchers to use machine learning regression models for survival analysis, further increasing the role of machine learning in the community.

One application of such survival analysis methods is heart failure risk prediction \citep{chicco2020machine, newaz2021survival, fahmy2021machine, panahiazar2015using}. Heart failure occurs when the heart loses the ability to relax or contract normally, 
leading to higher pressures within the heart or the inability to provide adequate output to the body.
Heart failure is projected to affect over 6 million individuals in the U.S., 
is often fatal (i.e., it is mentioned in 13.4\% of death certificates), and costs society over 30 billion dollars, 
all in the United States alone \citep{cdc_hf_importance}. 
Luckily, heart failure can often be prevented with preventive therapy 
if those at high risk can be identified in advance \citep{tsao2022heart}. 
Machine learning models can help identify those at high risk better than traditional survival analysis models by offering improved discrimination.
% a challenging problem using traditional approaches \citep{cleland2019prevention}.
Nonetheless, one of the main roadblocks preventing machine learning approaches from becoming widely adopted in medical practice, including in heart failure risk prediction, is their black-box nature.

% heart failure risk prevention \citep{cleland2019prevention}, including a lack of transparency and trustworthiness.
% Despite offering improved discrimination, black-box machine learning approaches have not yet become widely adopted for many medical problems, including heart failure risk prediction, due to their black-box nature.
Interpretable machine learning methods hold the promise of delivering both high accuracy and interpretability, and thus have the potential to promote the adoption of machine learning methods in practice.
Post-hoc explainability methods such as 
SHapley Additive exPlanations (SHAP) \citep{lundberg2017unified} and 
Local Interpretable Model-agnostic Explanations (LIME) \citep{ribeiro2016model} are widely used, but provide potentially limited intelligibility as post-hoc methods \citep{kumar2020problems, van2022tractability, alvarez2018robustness, rahnama2019study}. 
Further, Explainable Boosting Machines \citep{lou2013accurate} have gained popularity due to their inherent interpretability as a generalized additive model \citep{hastie2017generalized} and state-of-the-art accuracy.

To handle right-censored data, learn from time-to-event signal, and increase model interpretability, we present a complete pipeline to perform interpretable survival analysis without the need to estimate survival times. Specifically, we present an improved, scalable version of survival stacking \citep{craig2021survival}, which generates classification training samples from the risk sets of the underlying survival analysis problem without having to estimate survival times (like pseudovalues do).
Additionally, to reduce feature correlation that can hurt interpretability, we use ControlBurn \citep{liu2021controlburn} for (nonlinear) feature selection. ControlBurn prunes a forest of decision trees, thereby selecting important risk factors from a potentially large collection of candidate features.
After survival stacking and feature selection, we use Explainable Boosting Machines \citep{lou2013accurate} to generate trustworthy yet accurate survival predictions with both global and local explanations.

To evaluate our pipeline, we study heart failure risk prediction using  electronic health record (EHR) data from a large hospital network with over 350,000 patients. Our experimental results show that our pipeline can predict heart failure more accurately than traditional survival analysis models and is comparable to other state-of-the-art machine learning models, while providing intelligibility. Our models both validate known risk factors for incident heart failure and identify novel risk factors.

\section{Related Work}\label{sec:related-work}
% Make sure you also put your work in the context of related
% work.  Who else has worked on this problem, and how did they approach
% it?  What makes your direction interesting or distinct?

Several previous works have explored nonlinear extensions to classical survival models such as Cox proportional hazards models. 
A popular approach is to use the Cox partial likelihood as a loss function for common machine learning models, such as generalized additive models \citep{hastie1995generalized, utkin2022survnam}, boosting \citep{ridgeway1999state}, support vector machines \citep{van2007support}, and deep neural networks \citep{katzman2018deepsurv}. 
While generalized additive models are interpretable, using the Cox partial likelihood inherently enforces a proportional hazards assumptions, which might not be met in practice.
Other works discard the proportional hazards assumption and instead use machine learning methods to estimate the survival curve, popular examples including Random Survival Forests \citep{ishwaran2008random}, RNN-Surv \citep{giunchiglia2018rnn}, and DeepHit \citep{lee2018deephit}. However, these models' lack of interpretability is a challenge for implementation in clinical practice.

Other related works have explored casting survival analysis to a binary classification or regression problem.
Most notably, the use of pseudo-values \citep{andersen2010pseudo} enables converting a survival analysis problem to a regression problem by directly predicting Kaplan-Meier estimates of the survival curve.
Using pseudo-values does not require a proportional hazards assumption, and has been paired with various regression models in previous works \citep{rahman2021deeppseudo, zhao2019dnnsurv, rahman2022pseudo, rahman2022fair, feng2021bdnnsurv}.
Perhaps most similar to our paper, PseudoNAM \citep{rahman2021pseudonam} combines pseudo-values with Neural Additive Models \citep{agarwal2021neural} for interpretable survival analysis. 
While PseudoNAM provides a straightforward solution for interpretable survival analysis (but without potentially important interaction terms), 
it could be biased when censoring is not independent of the covariates \citep{binder2014pseudo},
a common situation in medical applications that censor on death.
Instead of pseudo-values, our paper uses survival stacking \citep{craig2021survival} to cast survival analysis as a classification problem, which avoids such bias. 
\section{Methods}\label{sec:methods}
% Tell us your techniques!  If your paper is develops a novel machine
% learning method or extension, then be sure to give the technical
% details---as you would for a machine learning publication---here and,
% as needed, in appendices.  If your paper is developing new methods
% and/or theory, this section might be several pages.

% If you are combining existing methods, feel free to cite other
% packages and papers and tell us how you put them together; that said,
% the work should stand alone for someone in that general machine
% learning area.  
% \emph{Lack of technical details, such that the soundness of the
%   methods can be verified, is a major reason that otherwise
%   strong-looking papers are scored low/rejected.}

We present a complete pipeline for interpretable and accurate survival analysis. We use improved survival stacking, feature selection via ControlBurn \citep{liu2021controlburn}, and interpretable classification via Explainable Boosting Machines \citep{lou2013accurate} for our pipeline, which is summarized in Figure~\ref{fig:main_fig}.
We present each part of our pipeline in more detail in the proceeding subsections.
For more background on survival analysis, see Appendix \ref{sec:surv_background}.

% Our methodology scales well to large-scale survival analysis tasks common in healthcare scenarios.

% One of our main contributions from a methods perspective is combining two fundamental ideas and methods: using interpretable models (e.g., EBMs and ControlBurn) and survival stacking \citep{craig2021survival} to yield interpretable survival analysis. To that end, we first cover the fundamentals that we build on specifically for EBMs in Section~\ref{sec:EBMs}, ControlBurn in Section~\ref{sec:ControlBurn}, and survival stacking in Section~\ref{sec:survival-stacking}. Then, we show how we construct the interpretable survival analysis pipeline in Section~\ref{sec:ISA_methods}.

\begin{figure*}[t]
    \centering
    \includegraphics[width=0.8\linewidth]{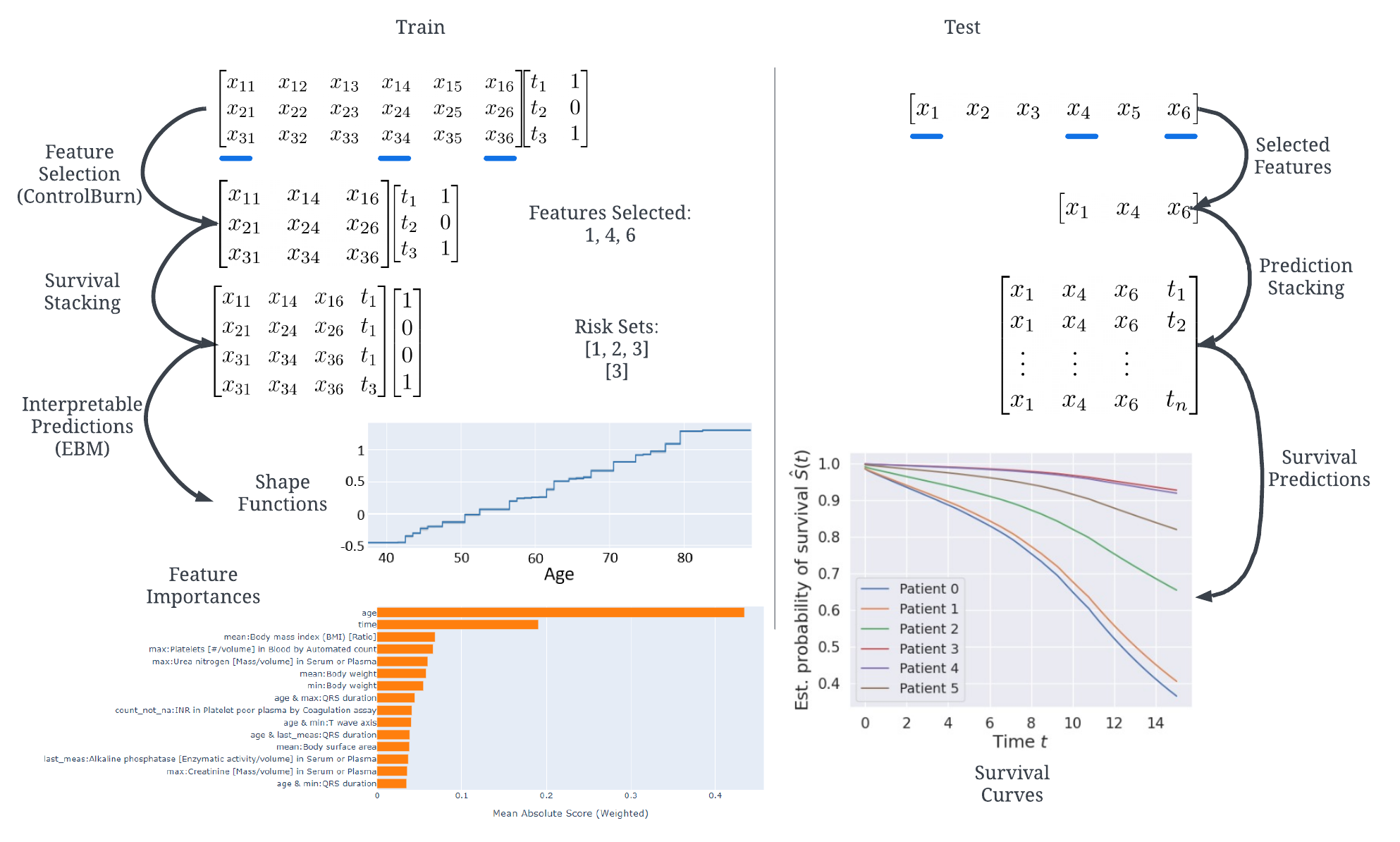}
    \caption{A summary of our interpretable survival analysis pipeline. During training (left), we use ControlBurn for feature selection (Section \ref{sec:ControlBurn}), survival stacking with subsampling to cast the survival data to classification data (Algorithm \ref{alg:survival_stacking}), and EBMs for generating feature importances and shape plots (Section \ref{sec:EBMs}). At test time (right), we use Survival Prediction (Algorithm \ref{alg:survival_prediction}) to generate survival curves using the trained ControlBurn and EBM models.}
    \label{fig:main_fig}
\end{figure*}

%%%%%%%%%%%%%%%%%%%%%%%%%%%%%%%%%%%%%%%%%%%%%%%% Survival Stacking
\subsection{Survival Stacking} \label{sec:survival-stacking}

% An intrinsic challenge of survival analysis is predicting the survival probability of each patient at \textit{all} future time points, rather than a single fixed time point.
% Proportional hazards models bypass this difficulty by making the strong (and often unwarranted) that assumption all survival curves are proportional to a single (learned) baseline survival curve.
% Our approach avoids estimating the survival curve directly. 
% Instead, we use survival stacking \citep{craig2021survival} to leverage existing interpretable machine learning classification models for survival analysis.
As discussed in Section \ref{sec:introduction}, survival analysis is often a better solution for risk prediction than fixed-time classification, as survival analysis naturally handles censoring and incorporates time-to-event signal. 
However, much more research in machine learning, including interpretable machine learning, has focused on classification models.
Thus, casting a survival analysis problem to an equivalent classification problem would enable the use of such state-of-the-art classification models.

We describe an improved version of survival stacking \citep{craig2021survival} that allows for the use of binary classification models for survival analysis.
We assume survival data $\{(X_i, T_i, \delta_i)\}_{i=1}^n$ where $X_i$ is a vector of covariates,
$T_i$ is the time when the event of interest occurs (e.g. getting heart failure), 
and $\delta_i$ is the censoring indicator, indicating whether, at time $T_i$, patient $i$ reaches the event of interest or is lost to future observation before developing the condition. 
For each event time $t$ observed in the original data set, survival stacking adds all samples $X_i$ in the corresponding risk set $R(t) = \{i : T_i \geq t\}$ to a new ``stacked'' data set.
% An additional 
%We add an additional feature to each risk set sample that indicates the corresponding event time.
For sample $i \in R(t)$, the corresponding binary label in the stacked data set is $0$ if $T_i > t$, and $1$ if $T_i= t$. 
Additionally, an extra covariate is added to the stacked data set representing the time $t$ which defines the risk set $R(t)$.
Survival stacking works because training a binary classifier on the survival stacked data estimates the hazard function $\lambda(t \mid X)$ (the instantaneous risk conditioned on surviving until time $t$), which can be used to generate survival curves (see Appendix \ref{sec:surv_background}).
An example of survival stacking on a smaller data set can be found in \citep{craig2021survival}. 

We use an improved version of survival stacking in our pipeline, which is outlined in Algorithm~\ref{alg:survival_stacking}. 
Specifically, we make two modifications to survival stacking as in \citep{craig2021survival} to better scale to large survival data sets.
First, instead of defining a one-hot-encoded categorical variable to represent time, we define a single continuous time feature. This choice reduces the number of features, aiding computational efficiency and interpretability.
%Plus, note that 
%Our version of survival stacking, which we call survival stacking with subsampling, is summarized in Algorithm \ref{alg:stacking}.
% Additionally, when using a tree-based model for classification as we do in this paper, defining a single continuous time feature is equivalent to ordinal encoding, which is often as effective as one-hot encoding for trees \mvnnote{find cite probably}.
Second, stacking increases the number of samples quadratically, potentially creating computational hardships as well as a severe class imbalance in the case of a high censoring rate. We perform random undersampling of the majority class samples $\{i : T_i > t\} \subset R(t)$ to mitigate this issue.
% To that end, we introduce a subsampling approach to reduce the number of 0-labeled samples generated for each risk set. Specifically, for each risk set $R(t)$, we add rows to the stacked data set for all samples $i$ such that $T_i = t$, and add rows for each sample $i$ such that $T_i > t$ with probability $\gamma$ for some $\gamma < 1$. 
% This is analogous to random undersampling of the majority class samples $\{i : T_i > t\} \subset R(t)$.

% creating the stacked data set results in a quadratic growth in the number of samples.
% This creates computational issues, as well as extreme class imbalance when there is a high rate of censoring, which is often the case in disease risk prediction as in our heart failure application.
% Thus, we introduce a subsampling approach to reduce the number of censored samples in each risk set and thus improve the class imbalance.
% In detail, for each risk set $R(t)$, we add rows for all \textit{uncensored} samples in $R(t)$, and we add rows for each \textit{censored} sample in $R(t)$ with probability $\gamma$ for some $\gamma < 1$. 
% This is analogous to random undersampling of the majority class samples in each risk set, a common strategy for handling imbalanced data sets.

\subsubsection{Survival Stacking Prediction}\label{sec:survival-stacking-prediction}

\begin{figure}[!h]
    \centering
    \includegraphics[width=\linewidth]{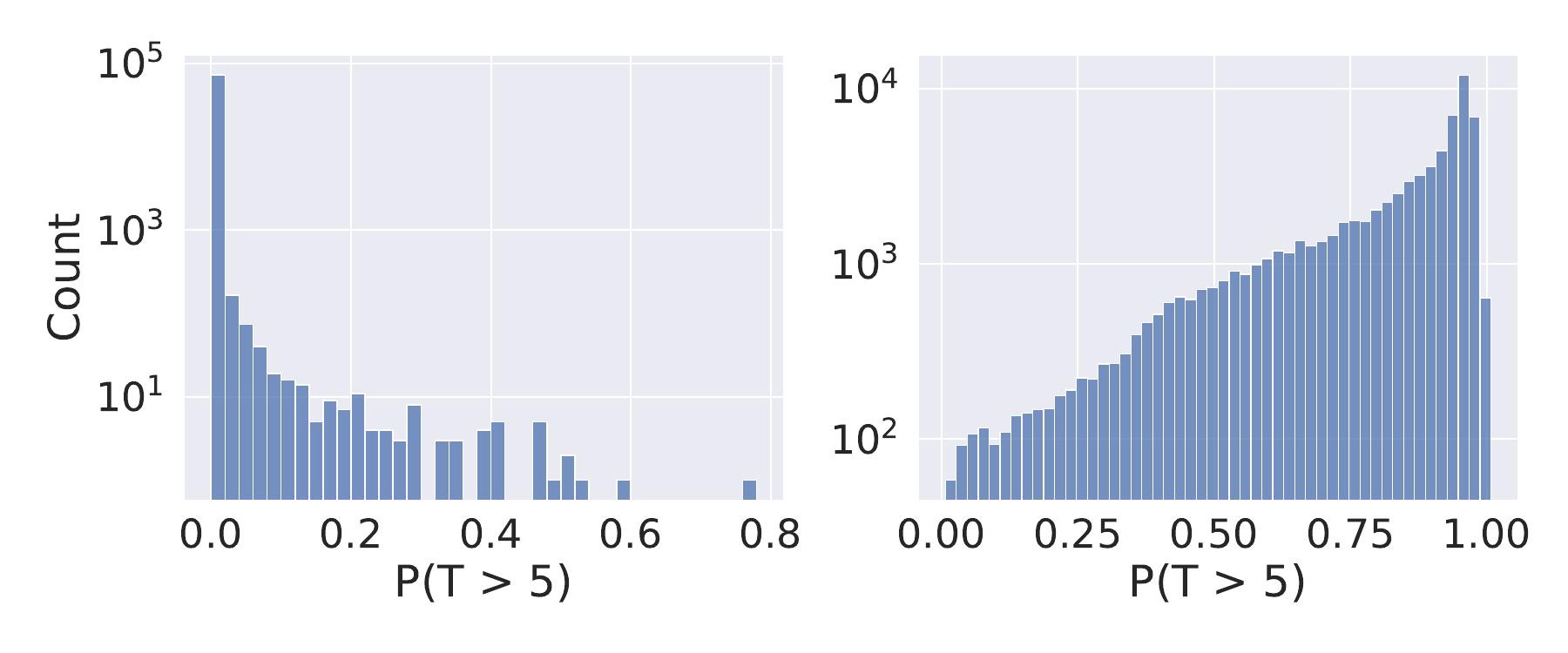}
    \caption{Distribution of predicted survival probabilities $S(t \mid X_i) = P(T_i > t)$ at $t=5$ across test patients. The left plot shows the survival probabilities using Eq.~\eqref{eq:surv_pred_craig} as in \citep{craig2021survival}, while the right plot shows Eq.~\eqref{eq:surv_pred_monte_carlo}.
    Our survival prediction method gives a reasonable distribution, while the method from \citep{craig2021survival} is incorrectly skewed torwards 0.
    }
    \label{fig:surv_predictions}
\end{figure}

After training a classification model $f$ on survival stacked data, we must convert the model's predicted probabilities to survival curves for patients during inference.
In \citep{craig2021survival}, the survival curve $S(t \mid X) = P(T > t \mid X)$ is estimated using
\begin{equation}
\label{eq:surv_pred_craig}
    \hat{S}(t \mid X) = \prod_{t_k \leq t} \Big( 1 - f(X \concat t_k)\Big).
\end{equation}
where $\concat$ represents concatenation to add the extra time covariate from survival stacking.
The motivation for Eq. \eqref{eq:surv_pred_craig} is that if the time variable $T$ is assumed to be discrete, taking on only the observed times $t_1, \ldots, t_k$ in the training set, then, assuming without loss of generality $t_1 \leq t_2 \leq \cdots \leq  t_k \leq t$, $S(t \mid X)$ can be written as the product of conditional survival probabilities up until time $t$:
\begin{equation}
\label{eq:surv_prob_discrete_times}
    S(t \mid X) = P(T > t \mid X) = \prod_{t_k \leq t} \Big( 1 - \lambda(t_k \mid X)\Big) 
\end{equation}
% \begin{align}
%     S(t &\mid X) = P(T > t \mid X) \\
%     &= P(T>t_1 \mid X) \prod_{i=1}^{k} P(T > t_{i+1} | T>t_{i}, X) \\
%     &= \prod_{t_k \leq t} \Big( 1 - \lambda(t_k \mid X)\Big),
% \end{align}
see Appendix \ref{}\citep{suresh2022survival} for details.
This motivates Eq.~\eqref{eq:surv_pred_craig} since the binary classifier $f$ estimates the hazard function $\lambda(t \mid X)$ with survival stacking.
% Similarly, the cumulative hazard function (CHF) $\Lambda$ under the discrete time assumption is given by $\Lambda(t \mid X) = \int_0^t \lambda(t \mid X) = \sum_{t_k \leq t} \lambda(t_k \mid X)$,
% yielding an alternative expression for $S(t \mid X)$:
% \begin{equation}
% \label{eq:surv_pred_from_chf}
%     S(t \mid X) = \exp \left( - \sum_{t_k \leq t} \lambda(t_k \mid X)\right).
% \end{equation}
However, the estimate in Eq.~\eqref{eq:surv_pred_craig} can become unstable in large sample, as demonstrated in Figure \ref{fig:surv_predictions}.
Further, we use a continuous time feature in our survival stacking algorithm, implying that this discrete time assumption may not be suitable. 
Thus, we propose a different method for predicting the survival curve, which is summarized in Algorithm \ref{alg:survival_prediction}. 
We first predict the cumulative hazard function $\Lambda(t \mid X) = \int_0^t \lambda(s \mid X) ds$ using Monte Carlo integration at $n$ uniform continuously sampled times $t_1, \ldots, t_n \leq t$:
\begin{equation}
\label{eq:chf_pred}
    \hat{\Lambda}(t \mid X) = \frac{t}{n} \sum_{i=1}^n f(X \concat t_i).
\end{equation}
After estimating $\Lambda(t \mid X)$, we can naturally estimate $S(t \mid X)$ using
\begin{equation}
\label{eq:surv_pred_monte_carlo}
    \hat{S}(t \mid X) = \exp\left(- \hat{\Lambda}(t \mid X) \right).
\end{equation}
% This estimate of $S(t \mid X)$ is similar to Eq.~\eqref{eq:surv_pred_from_chf}, but differs in that the scaling factor $\frac{t}{n}$ improves stability.
Using this estimator is possible since we used a continuous time feature in survival stacking, which allows for such Monte Carlo integration.

\begin{algorithm}
% \label{alg:stacking}
\caption{Survival Stacking With Subsampling}
\begin{algorithmic}[1]\label{alg:survival_stacking}
    \STATE \textbf{Input:} Survival data $(X_1, T_1, Y_1),  \ldots,$\\ \hspace{1.25cm}$(X_n, T_n, Y_n)$, sampling ratio $\gamma$.
    \STATE \textbf{Output: } Classification data.
    \STATE $\text{event\_times} \gets \{T_i : Y_i = 1\}$, $\text{samples} \gets [~]$.
    \FOR{$t$ in event\_times}
        \STATE samples $\mathrel{+}= $ $\{(X_i \concat T_i, 1) : T_i = t, Y_i = 1\}$.
        \STATE risk\_set = uniform random sample with \\ \hspace{1.52cm} probability $\gamma$ from $\{i : T_i > t\}$.
        \STATE samples $\mathrel{+}= \{(X_i \concat T_i, 0) : i \in \text{risk\_set}\}$.
    \ENDFOR
    \RETURN samples.
\end{algorithmic}
\end{algorithm}

\begin{algorithm}
\label{alg:prediction}
\caption{Survival Prediction}
\begin{algorithmic}[1]\label{alg:survival_prediction}
    \STATE \textbf{Input:} Fitted classification model $f$, survival data test sample $X$, prediction time $t$.
    \STATE \textbf{Output: } Estimated survival probability \\ \hspace{1.3cm}$\hat{S}(t \mid X) = P(T > t \mid X)$.
    \STATE Sample $t_1, \ldots, t_n$ uniformly from $(0, t]$.
    \STATE Estimate CHF via Monte Carlo integration:\\ \hspace{1.3cm} $\hat{\Lambda}(t \mid X) = \frac{t}{n} \sum_{i=1}^n f(X 
 \concat t_i)$. 
    \RETURN $\hat{S}(t \mid X) = \exp(- \hat{\Lambda}(t \mid X))$.
    
\end{algorithmic}
\end{algorithm}

% One of the main benefits of survival stacking is that classification has a much richer history than survival analysis, and survival stacking allows for the use of any classification model. 
% We take advantage of this by adopting a state-of-the-art interpretable classification model on our stacked data set, an Explainable Boosting Machine (EBM). 

%%%%%%%%%%%%%%%%%%%%%%%%%%%%%%%%%%%%%%%%%%%%%%%% Interpretable Models
\subsection{Explainable Boosting Machines} \label{sec:EBMs}
Explainable Boosting Machines (EBMs) \citep{lou2012intelligible} are specific instances of Generalized Additive Models (GAMs) with interaction terms \citep{hastie2017generalized, lou2013accurate}:
\begin{equation}
\begin{split}
    g(\mathbb{E}[y]) = \beta_0 &+ f_1(x_1) + f_2(x_2) + \cdots \\
    &+ f_p(x_p) + \sum_{i,j} f_{i,j}(x_i,x_j),
\end{split}
\end{equation}
where $g$ denotes a link function (e.g.~identity for regression tasks and logistic for classification). The $f_i$'s are called the univariate shape functions, or `main effects', whereas the $f_{i,j}$'s encode the interaction terms between features $x_i$ and $x_j$ and are known as the `interaction effects' or `2D shape functions'. EBMs fit these shape functions by applying cyclic gradient boosting on shallow decision trees, see \citep{lou2013accurate} for technical details. 
This process includes a crucial purification process \citep{lengerich2020purifying} that ensures that each $f_i(x_i)$ encodes the full and sole effects of feature $x_i$ to the target in the model, and similarly so for higher-order terms, so that they form a functional ANOVA decomposition. As a result, shape functions do not necessarily show marginalised effects, unlike partial dependence plots (PDPs). 
Without purification, encoding the sole effect is not guaranteed, as an identifiability issue would arise: the contribution of $x_i$ could be moved freely between its main effect and its interaction terms without changing the model predictions \citep{lengerich2020purifying}.
% Perhaps surprisingly, EBMs do not sacrifice performance \citep{lou2012intelligible} while catering interpretability, and are a state-of-the-art tree-based method for tabular data, exhibiting comparable performance to XGBoost, Deep Neural Networks, Random Forests, et cetera \citep{nori2021accuracy, lou2012intelligible, liu2020impact, kamath2021explainable}.

EBMs provide interpretability via the plotting of main effects and interaction terms, as well as feature importances by averaging the absolute contributions of a feature to the target over all samples \citep{nori2019interpretml}.
Perhaps surprisingly, EBMs achieve comparable performance to state-of-the-art tabular prediction methods while providing more interpretability \citep{nori2021accuracy, lou2012intelligible, kamath2021explainable}.
They have been applied to a wide variety of fields, including high-risk applications in healthcare \citep{lengerich2022death, bosschieter2022using, sarica2021explainable, qu2022using}.

% \paragraph{Feature importances}
% In addition to shape functions, EBMs also provide feature importances by averaging the absolute contributions of a feature to the target over all samples \citep{nori2019interpretml}. As such, EBM feature importances are `fixed-points' of SHAP \citep{lundberg2017unified}.

% \paragraph{Hyperparameters} We use \verb|outer_bags=25| and \verb|inner_bags=15| to ensure robustness, and further use \verb|max_bins=128, interactions=50, max_rounds=5000| for expressivity. 

%%%%%%%%%%%%%%%%%%%%%%%%%%%%%%%%%%%%%%%%%%%%%%%% ControlBurn
\subsection{Feature Selection} \label{sec:ControlBurn}
% but also have the property that variance over correlated features is distributed roughly evenly. Intuitively, this means that if some signal in the data could be learned by multiple, correlated features, the feature importance that could be gained (attributed to that signal) is split approximately evenly among these correlated features, unless the interaction terms between these features are present. \textbf{}

One challenge for interpretable classification models is feature correlation, especially when the data set is high-dimensional.
In particular, feature importance scores can be split between correlated features, resulting in potentially biased feature importance rankings \citep{liu2021controlburn}.
This problem is rather common in electronic health record (EHR) data, which often contains many highly correlated features.
For example, healthcare data typically includes features for a patient's height, weight, and BMI, which are highly correlated.
Additionally, generating multiple features from a single feature's time series, such as a patient's average and last measured value of a lab test, typically resulting in correlated features.
We thus consider feature selection as an important part in the interpretable survival analysis pipeline.

We choose to use ControlBurn \citep{liu2021controlburn} for feature selection, which is specifically designed to mitigate the bias induced by the correlated features.
ControlBurn builds a forest of shallow trees and uses a LASSO model \citep{tibshirani1996regression} to prune trees, keeping only the features that are left in the unpruned trees.
This process is similar to the traditional linear LASSO model, but is capable of capturing nonlinear relationships through ensembles of decision trees before pruning.
ControlBurn has been shown to be efficient and outperform other feature selection methods on data with correlated features \citep{liu2021controlburn, liu2022controlburn}, making it a good option for large-scale healthcare data sets. 

Since ControlBurn performs feature selection through classification, using the survival stacked data set discussed in Section \ref{sec:survival-stacking} to perform feature selection is a natural choice.
However, in some cases, it may be too computationally expensive to build the survival stacked data set before doing feature selection if both the number of samples and number of features are large.
In such large data cases, a reasonable alternative is to fix a future time $t$ and perform feature selection using classification data at the fixed time $t$.
\section{Experiments} 

\subsection{Data}

We evaluate our interpretable survival analysis pipeline through a large-scale study of incident heart failure risk prediction.
Specifically, we gather EHR data from a large-scale hospital network (name censored for anonymity) for a total of $n = 363,398$ patients and $p = 1,590$ features.
% including features for demographics, measurements (labs and vital signs), conditions, and medications.
Additional information about cohort selection and characteristics are provided in Appendix \ref{sec:cohort_selection}.

\subsection{Setup}
For preprocessing, we standardize continuous features across the observed entries and use mean imputation (equivalent to 0-imputation after standardization), while we use one-hot encoding for categorical features.
% Observe not all models need this, like tree-based models, but we use it anyway to maintain consistency and a fair comparison
Then, for feature selection, we contrast ControlBurn (discussed in Section \ref{sec:ControlBurn}) and linear LASSO \citep{tibshirani1996regression}.
% When using survival stacking, we use a subsampling rate of $\gamma = 0.01$.
We apply an 80/20 train-test split, and evaluate models with 5 trials each with different random seeds. The models we run are a Cox proportional hazards model (CoxPH, \citep{cox1972regression}), Random Survival Forest (RSF, \citep{ishwaran2008random}), Logistic Regression (LogReg), XGBoost \citep{chen2016xgboost}, and an Explainable Boosting Machine (EBM, \citep{lou2013accurate, nori2019interpretml}). Note that LogReg, XGBoost, and the EBM are fit on the stacked data built using Algorithm \ref{alg:survival_stacking}.
% For survival analysis, we consider the following models, where we use survival stacking to run all classification models:
% \begin{itemize}
%     \item \textbf{CoxPH}: Cox proportional hazards model \citep{cox1972regression} that assumes a log-linear hazard function.
%     \item \textbf{RSF}: Random Survival Forest \citep{ishwaran2008random}, a survival model based on random forests \citep{breiman2001random} that does \emph{not} assume proportional hazards.
%     \item \textbf{LogReg}: Logistic regression model fit on survival stacked data. This is approximately equivalent to the Cox model in Eq.~\eqref{eq:surv_pred_craig}, as proven by \citep{craig2021survival}.
%     \item \textbf{XGBoost}: State-of-the-art tree-based classification model \citep{chen2016xgboost}, fit on survival stacked data.
%     \item \textbf{EBM}: Explainable Boosting Machines \citep{lou2013accurate, nori2019interpretml}, a state-of-the-art interpretable machine learning model as discussed in Section \ref{sec:EBMs}, fit on survival stacked data.
% \end{itemize}
All models are evaluated through the cumulative/dynamic AUC and integrated Brier score as defined in scikit-survival \citep{polsterl2020scikit}. 
\subsection{Interpretable Survival Analysis Pipeline Performance}\label{sec:performance}

We now evaluate the predictive performance of our interpretable survival analysis pipeline, as summarized in Figure \ref{fig:main_fig}, showing the results in Table \ref{tab:accuracy_scores}.

\begin{table*}[!h]
    \caption{A comparison of survival analysis models in terms of \textit{time-dependent} AUC and Brier score for various feature selection strategies. Logistic regression, XGBoost, and EBM models are run using survival stacking as described in Section \ref{sec:methods}. Errors denote standard deviations over 9 trials across 3 different feature sets generated through different random seeds.}
    \centering

    \scalebox{0.75}{
    \begin{tabular}{l l |c c c c c c}
    \toprule
          \textbf{Metric} & \textbf{Feature Selection} & \textbf{CoxPH} & \textbf{CoxPH + Int} & \textbf{RSF} & \textbf{LogReg} & \textbf{XGBoost} & \textbf{EBM} \\
         \midrule
        \multirow{4}{*}{\textbf{AUC}}
        & Lasso (k = 10) &  $0.799 \pm 0.005$ & $0.817 \pm 0.000$ & $0.807 \pm 0.003$ & $0.804 \pm 0.002$ & $0.817 \pm 0.001$ & $\bm{0.821 \pm 0.001}$ \\
        & Lasso (k = 50) & $0.815 \pm 0.003$ & - & $0.805 \pm 0.009$ & $0.815 \pm 0.002$ & $0.829 \pm 0.001$ & $\bm{0.834 \pm 0.001}$ \\
       & ControlBurn (k = 10)  &  $0.799 \pm 0.005$ & $0.810 \pm 0.003$ & $0.812 \pm 0.004$ & $0.799 \pm 0.002$ & $\bm{0.819 \pm 0.002}$ & $\bm{0.823 \pm 0.002}$ \\
       & ControlBurn (k = 50) &  $0.814 \pm 0.005$ & & $0.806 \pm 0.017$ & $0.814 \pm 0.003$ & $\bm{0.832 \pm 0.001}$ & $\bm{0.834 \pm 0.001}$ \\
       \midrule 
       %  \multirow{4}{*}{\textbf{C-Index}} & Lasso (k=10) & 0.804 & $0.808 \pm 0.000$ & 0.797 & 0.804 & 0.821 & 0.816 \\
       %  & Lasso (k=50) & 0.816 & - & 0.740 & 0.814 & 0.829 & 0.827 \\
       %  & ControlBurn (k=10) & 0.799 & $0.795 \pm 0.002$ & 0.790 & 0.786 & 0.811 & 0.811 \\
       % & ControlBurn (k=50) & 0.814 & - & 0.751 & 0.814 & 0.830 & 0.826 \\
       % \midrule
       \multirow{4}{*}{\textbf{Brier}} & Lasso (k = 10) &  $\bm{0.033 \pm 0.002 }$ & $\bm{0.034 \pm 0.000}$ & $\bm{0.034 \pm 0.001}$ & $\bm{0.036 \pm 0.001}$ & $\bm{0.036 \pm 0.001}$ & $\bm{0.036 \pm 0.001}$ \\
        & Lasso (k = 50)    &  $\bm{0.033 \pm 0.001}$ & - & $\bm{0.035 \pm 0.001}$ & $\bm{0.035 \pm 0.001}$ & $\bm{0.035 \pm 0.001}$ & $\bm{0.035 \pm 0.001}$ \\
       & ControlBurn (k = 10) &  $\bm{0.034 \pm 0.002}$ & $\bm{0.034 \pm 0.000}$ & $\bm{0.034 \pm 0.001}$ & $\bm{0.036 \pm 0.001}$ & $\bm{0.036 \pm 0.001}$ & $\bm{0.035 \pm 0.001}$ \\
       & ControlBurn (k = 50) &  $\bm{0.033 \pm 0.001}$ & - & $\bm{0.034 \pm 0.001}$ & $\bm{0.036 \pm 0.001}$ & $\bm{0.035 \pm 0.001}$ & $\bm{0.035 \pm 0.001}$ \\
    \end{tabular}
    }
    \label{tab:accuracy_scores}
\end{table*}

% The performance gap between 10 and 50 features is actually rather limited, suggesting some features are of critical importance. While the Brier scores for all models are similar, the EBM and XGBoost models seem to attain the highest AUC scores, which is noteworthy since this reaffirms that the EBM does not sacrifice accuracy for explainability. These interpretability aspects are further discussed in Section~\ref{sec:interpretable-results}.

There are several key observations from Table~\ref{tab:accuracy_scores}. First, there is a noticeable increase in AUC going from 10 to 50 features.
%, while all Brier scores are similar.
This demonstrates that using more features significantly boosts model discrimination, even though smaller feature sets are more often used in clinical practice.
For AUC, EBMs achieve the best performance, slightly better than XGBoost.
This aligns with previous research that suggests that EBMs achieve comparable performance with state-of-the-art models.
Additionally, this demonstrates that our survival stacking approach can be used with classification models to achieve performance at least as good as the performance of survival models.
For feature selection, in conjunction with such state-of-the-art models, ControlBurn slightly outperforms LASSO.
For an additional experiment comparing ControlBurn and LASSO, see Appendix \ref{sec:feature-selection}.
% , aligning with the results from Figure \ref{fig:num_features-to-AUC}.
Lastly, all models have small and roughly equivalent Brier scores, indicating good calibration. 
% This is further demonstrated in the calibration plots in the Appendix. 

% Observe that EBMs achieve best performance, slightly better than XGBoost. 

% - All Brier scores are comparable
% - For AUC, EBM is achieving best performance, even better than RSF and XGBoost.
% - Going from 10 to 50 features adds a noticeable increase in AUC (except for RSF). 

% - In conjunction with state-of-the-art models, ControlBurn slightly outperforms LASSO.

%%%%%%%%%%%%%%%%%%%%%%%%%%%%%%%%%%% Interpretable Results
\subsection{Interpretability Results} \label{sec:interpretable-results}
We present interpretable results generated by the EBM model after using ControlBurn to select $50$ features.
% We generate interpretable results by training an EBM on 50 \mvnnote{Should be 51 I think?} features selected by ControlBurn. 
The 15 most important features, along with their importance scores, are shown in Figure~\ref{fig:EBM-ft-imp} in Appendix \ref{sec:appendix_exp_details}. The ``time" feature represents the time variable generated during the survival stacking, defining the risk sets. 

% int plots always w/ age =>  exacerbated by

\subsubsection{Individual Shape Functions}
We show the shape functions of age, BMI, max serum creatinine value, and the time variable in Figure~\ref{fig:shapefuncs}. Each shape function represents the corresponding feature's individual contribution to heart failure risk hazard on a log-odds scale, adjusted for all other features. Note that an EBM's (i.e., a GAM's) shape functions differ from partial dependence plots as generated by e.g.~SHAP \citep{lundberg2017unified}, which performs marginalization. We make several key observations from the shape functions in Figure~\ref{fig:shapefuncs}, which we discuss one by one.
\begin{figure*}
    \centering
    \includegraphics[width=0.7\linewidth]{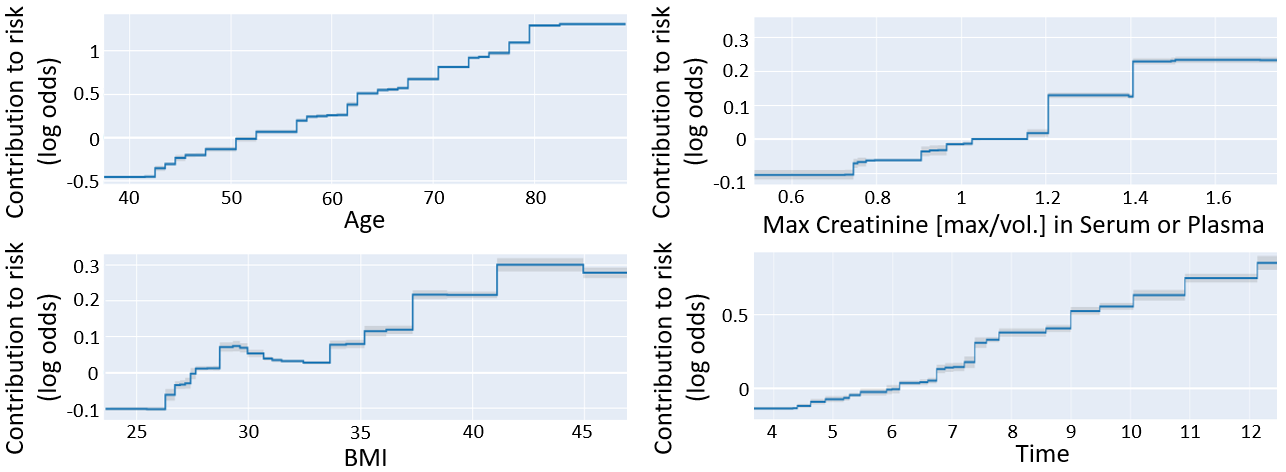}
    \caption{The EBM's shape functions for Age, BMI, max creatinine value, and time.}
    \label{fig:shapefuncs}
\end{figure*}

\paragraph{Shape function of Age}
The (log odds) contribution of age to the HF hazard appears to be a near-linear function, steadily increasing from $-0.5$ to $>1.3$. This increase is very significant on a log-odds scale, and further highlights the importance of Age (along with its large feature importance in Figure~\ref{fig:EBM-ft-imp} in the Appendix). This is aligned with prior data on the association between age and heart failure incidence \citep{khan201910}.

\paragraph{Shape function of BMI}
Perhaps surprisingly, the shape function of BMI is not strictly increasing, but seems to slightly drop after hitting a BMI of 30. One possible explanation might be that patients with a BMI of 30 or over are often advised to make lifestyle changes (e.g., increased exercise) upon being classified as obese (BMI threshold of 30), perhaps slightly mitigating the adverse effects of a high BMI. However, the risk seems to increase monotonically again once a BMI of 33 is reached. 

\paragraph{Shape function of Creatinine} Given that high creatinine values indicate abnormal renal function, and that kidney dysfunction is a major risk factor for heart failure, a monotonically increasing function might be expected. We indeed observe such monotonicity, although it is unclear why there appear to be increases in risk at the specific values 1.2 and 1.4.

\paragraph{Shape function of Time} While the shape function for Time might not yield a direct clinical interpretation, it does demonstrate the effects of survival stacking. 
Recall that the size of $R(t) = \{j : T_j \geq t\}$ monotonically decreases with respect to $t$; thus, the probability of being the next patient to get heart failure, i.e.~the hazard function, should increase over time.

\subsubsection{Interaction Terms}

%%%% Interaction terms

In addition to the individual shape functions, we also find that there are very strong interaction terms influencing one's risk of heart failure. 
In Figure~\ref{fig:shapefuncs-interactions}, we plot two interaction terms with noticeably strong signal.
Each interaction plot is a heatmap, showing regions in the interaction space that most strongly contribute to prediction.

First, age and the QRS duration (i.e., the time interval for the heart ventricles to be electrically activated) are important features on their own (evidenced by the fact they are selected during feature selection), but their interaction is also crucial: having a wide QRS is a stronger risk factor for younger people than older people. There are several biologically plausible mechanisms to explain this observed interaction.

For one, widening of the QRS complex can occur during either normal aging of the heart's conduction system or with progressive ventricular enlargement and/or dysfunction. Thus, QRS prolongation may be more closely related to the pathophysiologic progression of heart failure in younger patients. Alternatively, the predominant phenotype of heart failure transitions from heart failure with a reduced to a preserved ejection fraction with increasing age. QRS duration has been more clearly linked to morbidity and mortality in heart failure with a reduced ejection fraction and is even a viable therapeutic target (i.e., cardiac resynchronization therapy).

Second, the interaction term between age and the T-wave axis appears important for prediction. Here, the abnormal electrical signal (the T-wave axis) is more predictive among younger patients. Note that while the QRS duration and T-wave axis are related, all interaction terms are also corrected for all other individual and interaction terms. We pick the last measured values for the QRS duration and T-wave axis given that these are often used in practice.

% Widening of the QRS can occur during either normal aging of the heart's electrical system or with progressive enlargement and dysfunction of the heart. Therefore, widening of the QRS may be more related to the heart failure disease process in younger patients.

% Secondly, the interaction term between age and the T-wave axis appears important for prediction. Here again, the abnormal electrical signal (the T-wave axis) is more predictive among younger patients. Note that while the QRS duration and T-wave axis are related, all interaction terms are also corrected for all other individual and interaction terms. We pick the last measured values for the QRS duration and T-wave axis given that these are often used in practice. 
% \mvnnote{Add reference to Appendix if we add any additional interaction plots to the appendix.}

\begin{figure}
    \centering
    \includegraphics[width=1.\linewidth]{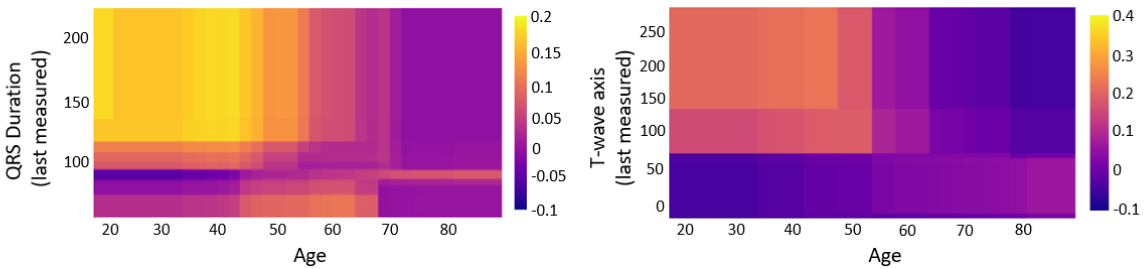}
    \caption{The EBM's shape functions for the interaction terms between (1) age and the QRS duration last measured, and (2) age and the T-wave axis last measured.}
    \label{fig:shapefuncs-interactions}
\end{figure}

%%%%%%%%%%%%%%%%%%%%%%%%%%%%%%%%%%%%%%%%%%%%%%%% DISCUSSION
\section{Discussion} \label{sec:discussion}
% \emph{This is probably the most important section of your paper!  This
%   is where you tell us how your work advances our understanding of
%   machine learning and healthcare.}  Discuss both technical and
% clinical implications, as appropriate.
\subsection{Machine learning insights}
The results in Table~\ref{tab:accuracy_scores} seem to suggest that state-of-the-art classification models can outperform `traditional' survival models in terms of AUC when included in our pipeline, while the Brier scores are comparable. Furthermore, the EBM seems to slightly outperform XGBoost, while the EBM is an interpretable model, suggesting the EBM might be an appropriate choice for survival stacked data as well as tabular data more generally. 
This suggests that the pipeline encompassing (efficient) survival stacking, feature selection, and an interpretable model with interaction terms is appropriate for healthcare applications.

\subsection{Clinical insights}
The ControlBurn feature selection model as well as the EBM feature importances shed light onto the most important risk factors for heart failure.
First, we find that age is the most important risk factor, aligning with existing clinical understanding.  Our study also identifies novel risk factors for incident heart failure. We found alkaline phosphatase levels and platelet count, both markers of liver dysfunction, are also predictors of increased risk of heart failure. Other important risk factors are BMI, urea nitrogen, QRS duration, T-wave axis, and creatinine. 

We also identify novel interactions between age and electrocardiographic (EKG) features (QRS duration and T-wave axis) that have not been identified in prior heart failure literature (as far as we are aware).  The interactions between age and these EKG features suggest that abnormal EKG results may be predictive of increased heart failure risk when present in younger patients. 

Lastly, and perhaps surprisingly, ControlBurn did not pick up gender and race as important risk factors. This is interesting because this might suggest that gender and race have limited signal for predicting heart failure after controlling for other clinical features that may be more proximal to incident heart failure.

% Recall that ControlBurn aims to assign the feature importance from a group of significantly correlated features to a single feature whenever possible, and removing every feature that is only correlated with other features but contains no signal in and of itself given these other features. Intuitively, if a feature is sufficiently correlated with many semi-important features but might have no inherent signal to heart failure prediction (e.g., often socio-demographic factors), ControlBurn should select that feature only when very few features are to be selected, since these `aggregate' features are correlated with heart failure. However, if more features are allowed, ControlBurn should select the perhaps slightly lesser-important features that carry signal, rendering the aggregate features expendable. Thus, in any case, prediction performance should increase as the number of features chosen increases, which is confirmed in Table~\ref{tab:accuracy_scores}.

%have features that determine whether someone gets kidney failure => kidney failure doesn't show up (just like gender, race, ...).

% Moreover, observe that in the shape function of BMI in Figure~\ref{fig:shapefuncs} there is a jump in HF risk hazard at a BMI of $\sim${}30 kg/m$^2$, which might be due to clinical interventions, since it seems unlikely this is a biological phenomenon. 

%%%%%%%%%%%%%%%%%%%%%%%%%%%%%%%%%%%%%%%%%%%%%%%% Limitations
\paragraph{Limitations}
% Explain when your approach may not apply, or things you could not
% check.  \emph{Discussing limitations is essential.  Both ACs and
%   reviewers have been advised to be skeptical of any work that does
%   not consider limitations.}

% Is our data set diverse? Predominantly caucasian
% Temporal data? Did we include full temporal data, or just statistics like mean, stddev, etc.?
% correlation vs causation
% need very granular data and a patient's info => semihard to use clinically without that

While the addition of subsampling helps survival stacking better scale to large data sets, there is still a significant challenge in terms of memory when building the survival stacked data set.
In our heart failure experiments, our training survival stacked data set has approximately 20.7 million rows from an initial training cohort of about 290,000 patients when using $\gamma = 0.01$ for subsampling.
For data sets larger than ours, it is possible that the data would have trouble fitting into memory without significantly more computational resources, in which case additional methodology such as mini-batching might be helpful. Additionally, we note that our data set comes from a single healthcare system; we have added relevant data characteristics in Table~\ref{tab:cohort_summary} for reference. Our results could be more robustly tested through experiments on additional healthcare systems.

%%%%%%%%%%%%%%%%%%%%%%%%%%%%%%%%%%%%%%%%%%%%%%%% CONCLUSIONS
\section{Conclusions}
We propose a novel pipeline for interpretable survival analysis that includes efficient survival stacking, feature selection through ControlBurn, and interpretable yet accurate predictions through EBMs. We show that this pipeline outperforms survival models such as the Cox model and Random Survival Forests, while also yielding interpretable results including feature importances and shape functions describing the individual and joint contributions of features to prediction. Our pipeline validates current understanding of heart failure, while also identifying novel risk factors for incident heart failure. We hope that our pipeline is useful to the community for healthcare applications more broadly.

% ACKNOWLEDGEMENTS ONLY GO IN THE CAMERA-READY, NOT THE SUBMISSION
% We have to ACK the OMOP data: https://med.stanford.edu/starr-omop/access.html with specific wording.

% \acks{This research used data or services provided by STARR, “STAnford medicine Research data Repository,” a clinical data warehouse containing live Epic data from Stanford Health Care (SHC), the Stanford Children’s Hospital (SCH), the University Healthcare Alliance (UHA) and Packard Children's Health Alliance (PCHA) clinics and other auxiliary data from Hospital applications such as radiology PACS. STARR platform is developed and operated by Stanford Medicine Research IT team and is made possible by Stanford School of Medicine Research Office}

\bibliography{refs}

\newpage
\appendix

%%%%%%%%%%%%%%%%%%%%%%%%%%%%%%%%%%%%%%%%%%%%%%%% Survival Analysis
\section{Survival Analysis}
\label{sec:surv_background}

Survival analysis aims to estimate the distribution of a time-to-event variable $T$ for some event of interest.
In the context of healthcare, survival analysis typically involves predicting if and when a patient will develop some condition or disease.
Formally, survival data for patient $i$ comes in the form $(X_i, T_i, \delta_i)$, 
where $X_i$ is a vector of covariate values,
$T_i$ is the time when the event occurs, 
and $\delta_i$ is the censoring indicator, 
which indicates if, at time $T_i$, patient $i$ develops the condition of interest or is lost to future observation without developing the condition. 
Survival analysis aims to estimate the probability that patient $i$ survives 
(that is, does not develop the condition of interest) by time $t$, 
$S(t \mid X_i) = P(T_i > t \mid X_i)$.
Estimating $S(t \mid X_i)$ is important for healthcare, 
as it can identify at-risk patients and help healthcare professionals provide appropriate treatments or risk reduction strategies.

Survival models typically aim to predict the hazard function $\lambda(t \mid X_i)$, which represents the instantaneous risk (probability) at time $t$ given $T_i > t$. 
For example, Cox models fit a log-linear function for the hazard function:
\begin{equation}
    \lambda(t \mid X_i) = \lambda_0(t) \exp(X_i^T \beta).
\end{equation}
After estimating the hazard function, the survival probability $S(t \mid X_i)$ can be estimated by using the relationship
\begin{equation}
\label{eq:surv_prob_exp}
    S(t \mid X_i) = \exp(- \Lambda(t \mid X_i)),
\end{equation}
where $\Lambda(t \mid X_i) = \int_0^t \lambda(s \mid X_i) ds$ is the cumulative hazard function. See \citep{jenkins2005survival} for more details.

\subsection{Survival Curves With Discrete Times}
\label{sec:surv_curve_discrete_times}

The survival probability $S(t \mid X) = P(T > t \mid X)$ represents the probability that a patient survives past time $t$, typing assuming that $T$ is continuous. 
However, sometimes it is reasonable to assume that $T$ is discrete, i.e. that patients can only reach the event of interest at a finite set of times. 
This is sometimes referred to as discrete-time modeling, see \citep{suresh2022survival}.

In such cases, we can derive $S(t \mid X)$ as a function of the hazard function $\lambda$ without the integral present in Eq.~\eqref{eq:surv_prob_exp}.
The result uses the following theorem:
\begin{theorem}
\label{thm:appendix_thm}
    Let $X$ be a discrete random variable, and let $t_1 \leq t_2$ be in the support of $X$. Then
    \begin{equation}
        P(X > t_2) = P(X > t_2 \mid X > t_1) P(X > t_1)
    \end{equation}
\end{theorem}
\begin{proof}
    \begin{align}
        P(X > t_2 &\mid X > t_1) P(X > t_1) \\
        &= \frac{P(X > t_2, X > t_1)}{P(X > t_1)} P(X > t_1) \\
        &= P(X > t_2)
    \end{align}
\end{proof}
Now, let $T$ have finite support, then 
by recursively applying Theorem \ref{thm:appendix_thm} to all times $t_1, \ldots, t_k \leq t$ in the support of $T$, we have
\begin{align}
    &S(t \mid X) = P(T > t \mid X) \\
    &= P(T>t_1 \mid X) \prod_{i=1}^{k} P(T > t_{i+1} | T>t_{i}, X)
    % &= \prod_{t_k \leq t} \Big( 1 - \lambda(t_k \mid X)\Big),
\end{align}
Additionally, when $T$ has finite support, the hazard function can be written as
\begin{equation}
    \lambda(t \mid X) = P(T = t \mid T \geq t, X).
\end{equation}
Thus, we have
\begin{align}
    &S(t \mid X) = P(T > t \mid X) \\
    &= P(T>t_1 \mid X) \prod_{i=1}^{k} P(T > t_{i+1} | T>t_{i}, X) \\
    &= \prod_{i=1}^{k} P(T > t_{i} | T\geq t_{i}, X) \\
    &= \prod_{i=1}^k \Big( 1 - \lambda(t_i \mid X)\Big),
\end{align}
as in Eq.~\eqref{eq:surv_prob_discrete_times}. For further investigation, see \citep{suresh2022survival}.

%%%%%%%%%%%%%%%%%%%%%%%%%%%%%%%%%%%%%%%%%%%%%%%% COHORT
\section{Cohort Details}
\label{sec:cohort_selection}

\begin{figure*}[t]
    \begin{minipage}{\linewidth}
        \centering
        \includegraphics[width=0.8\linewidth]{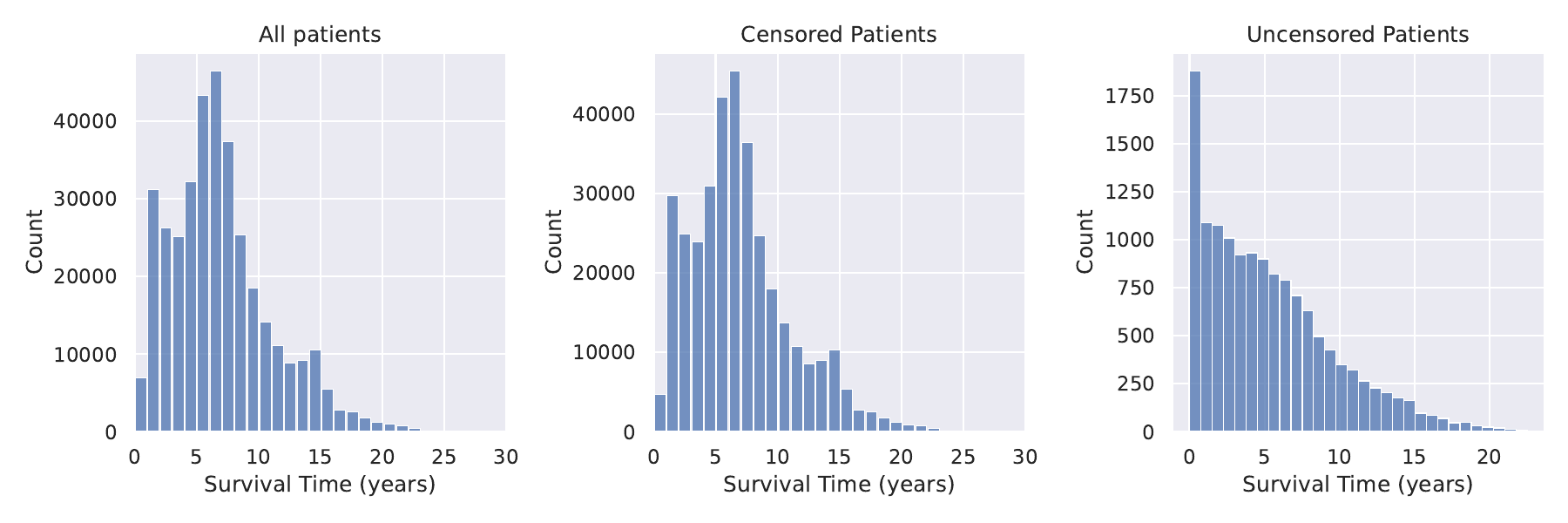}
        % \caption{Distribution of survival time for patients in our cohort, for all patients, censored patients, and uncensored (heart failure) patients.}
        \label{fig:time_hists}
    \end{minipage}

    \begin{minipage}{\linewidth}
        \centering
        \includegraphics[width=0.8\linewidth]{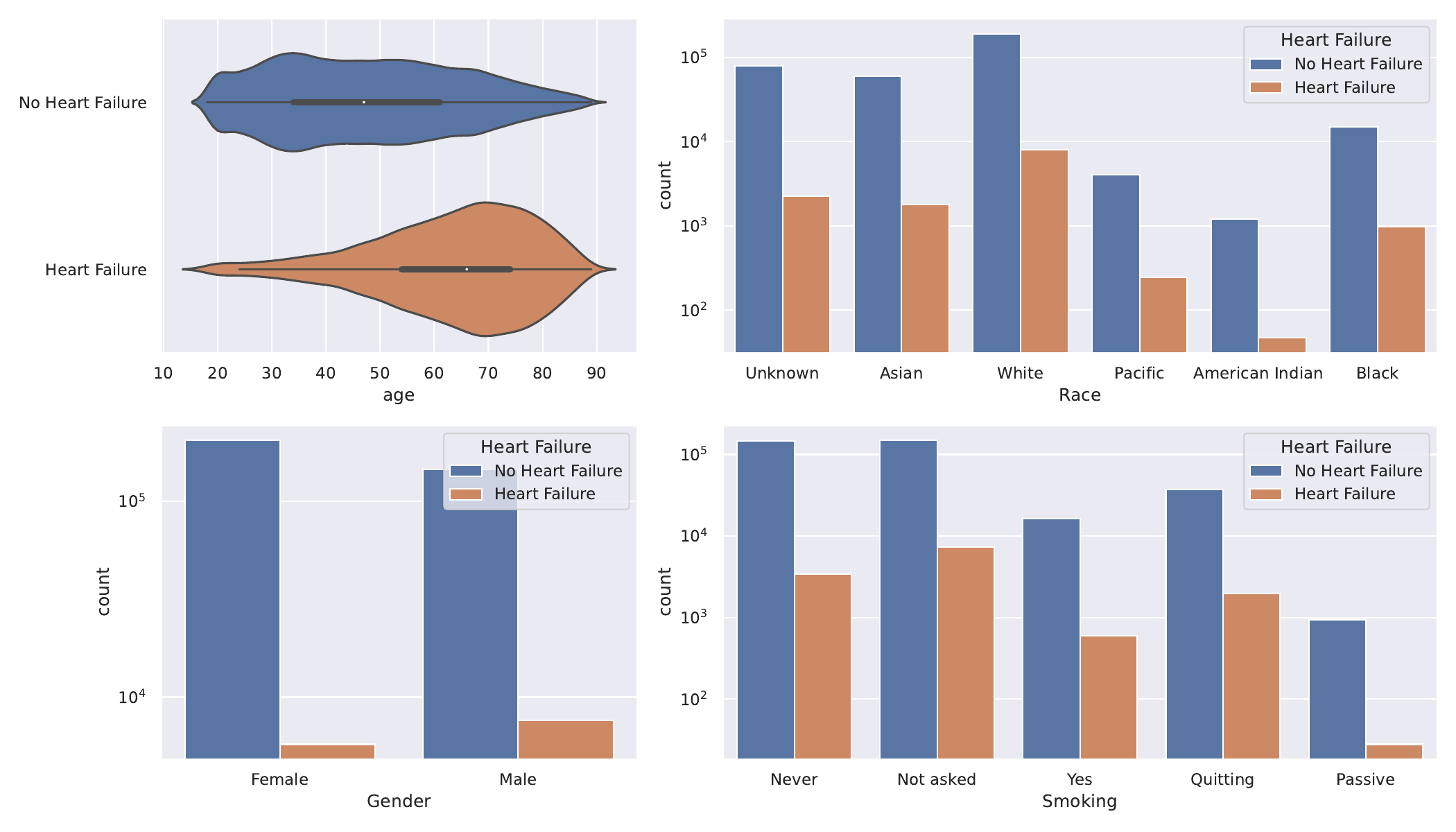}
        % \caption{Plots showing the relationship between heart failure and the demographic features in the cohort.}
        \label{fig:cohort_feature_plots}
    \end{minipage}

    \caption{Additional figures to illustrate cohort characteristics. \textbf{Top}: distribution of survival time for patients in our cohort, for all patients, censored patients, and uncensored (heart failure) patients. \textbf{Bottom}: plots showing the relationship between heart failure and the demographic features in the cohort.}
    \label{fig:extra_cohort_plots}
\end{figure*}

We evaluate our interpretable survival analysis pipeline through a large-scale study of incident heart failure risk prediction.
Specifically, we gather EHR data from a large-scale hospital network (name censored for anonymity) for a total of $n = 363,398$ patients and $p = 1,590$ features.
% including features for demographics, measurements (labs and vital signs), conditions, and medications.
Data characteristics of our cohort are given in Table~\ref{tab:cohort_summary}, and an additional cohort exploratory data analysis is in Appendix \ref{sec:appendix_cohort_details}.

\begin{table*}[h]
    \centering
    \caption{Heart Failure Cohort Statistics. More details on each individual feature is found in Table~\ref{tab:cohort_stats_all_fts}.}
    \vspace{0.2cm}
    \scalebox{0.8}{
    \begin{tabular}{l|l|l|l|l}
    \toprule
    \textbf{Summary}            & \textbf{Total Features By Type}           & \textbf{Gender}         & \textbf{Race}             & \textbf{Age}                     \\
    \midrule
    $n = 363398$       & Measurements: 1203 & Male: 42.1\%   & White: 54.6\%   &  18-29: 15.8\%           \\
    $p = 1590$         & Conditions: 115    & Female: 57.9\% & Asian: 17.0\%   & 30-44: 28.5\%           \\
    HF Prevalence: 3.7\% & Drugs: 268         &                & Black: 4.4\%    &  45-59: 26.6\%         \\
                       & Demographics: 4         &                & Other: 1.5\%   & 60-74: 20.8\%         \\
                       &                    &                & Unknown: 22.5\% & $\geq 75$: 8.3\% \\
    \bottomrule
    \end{tabular}
    }
    \label{tab:cohort_summary}
\end{table*}

\subsection{Filtering and Censoring}

We consider patients above the age of 18 with at least 3 years of continuous observation in the healthcare system. 
We set the index date, i.e.~the prediction date, for each patient as 2 years after the start of continuous observation, and use any data prior to the index date as the raw input data. 
Additionally, we only consider patients with an index date before January 1, 2018 to allow for a sufficient follow-up period. 
Lastly, we exclude patients who developed HF prior to their index date.

For censoring, we right-censor patients at the time of death as well as at the last known encounter time if one of these occurs before a heart failure diagnosis. 
Additionally, since we set the index date as the earliest date with 2 years of prior observation, our cohort includes patients with long observation periods.
Therefore, we also apply right-censoring at 15 years for patients that have not already been censored or had heart failure, which censors an additional 4.5\% of patients in our cohort.

\subsection{Data Extraction and Features}

The EHR data that we use to construct our cohort is stored in the OMOP Common Data Model (CDM), an open community data standard for storing EHR data \citep{hripcsak2015observational}.
Thus, our data extraction pipeline is easily applicable to any EHR database stored using the OMOP CDM.
% We define target and outcome cohorts in the OHDSI Atlas tool, and borrow code from PyHealth \citep{pyhealth2022github} to extract features from raw OMOP CDM data.
Below we list the feature types that we extract, along with the OMOP CDM table that the features come from.
\begin{itemize}
    \item \textbf{Measurements}: vital signs and lab measurements from the measurements table. For each measurement, we generate features for the min, max, mean, standard deviation, last observed value, and count of observed values over the patient's historical data prior to the index date.
    \item \textbf{Conditions}: binary features indicating whether or not a patient has had the given condition at any point before the index date, coming from the condition occurrence table.
    \item \textbf{Medications}: similar to conditions, but for prescriptions and over-the-counter drugs, coming from the drug exposure table.
    \item \textbf{Demographics}: features include the patient's age, gender, and race from the person table, and smoking history from the observation table.
\end{itemize}

Altogether, this yields 1590 features observed in the data set. These features capture current as well as historical data for patients at their index dates, which is an improvement over heart failure models that only use patient data at the index date \citep{khan201910, ahmad2017survival}.

\subsection{Heart Failure Labeling}

We generate survival labels $(T_i, \delta_i)$ for each patient $i$. 
The time $T_i$ represents the patient's event time, which is either the time the patient gets heart failure ($\delta_i = 1$), or the time they are lost to followup ($\delta_i = 0$). We determine whether or not a patient got heart failure using a curated OMOP CDM concept set, which defines all concepts in the raw data associated with heart failure diagnosis. 
This concept set was verified by clinicians, but given the large sample size, a small mislabeling risk cannot be ruled out.
% ; nonetheless, we acknowledge a small mislabelling risk without performing individual case analysis.
% \mvnnote{See if Alex agrees with this wording.}

\subsection{Cohort Statistics}
\label{sec:appendix_cohort_details}
For a list of all cohort features and their $10^{\textnormal{th}}, 25^{\textnormal{th}}, 50^{\textnormal{th}}, 75^{\textnormal{th}}$, and $90^{\textnormal{th}}$ percentiles, please see Table~\ref{tab:cohort_stats_all_fts}. 
Additionally, in Figure~\ref{fig:extra_cohort_plots}, we show additional plots to visualize the distribution of survival time as well as the demographic features in the cohort. 

% \twocolumn

% \begin{figure*}
%     \begin{minipage}{0.5\textwidth}
%         \centering
%         \includegraphics[width=\linewidth]{figures/survival_time_plots.pdf}
%         % \caption{Distribution of survival time for patients in our cohort, for all patients, censored patients, and uncensored (heart failure) patients.}
%         \label{fig:time_hists}
%     \end{minipage}%
%     \begin{minipage}{0.5\textwidth}
%         \centering
%         \includegraphics[width=\linewidth]{figures/cohort_feature_plots.pdf}
%         % \caption{Plots showing the relationship between heart failure and the demographic features in the cohort.}
%         \label{fig:cohort_feature_plots}
%     \end{minipage}
%     \caption{Side-by-side figures}
% \end{figure*}

% \twocolumn
%%%%%%%%%%%%%%%%%%%%%%%%%%%%%%%%%%%%%%%%%%%%%%%%%%%%%%%%%%% APPENDIX C
\section{Additional Experiment Details}
\label{sec:appendix_exp_details}

We use the following packages for our experiments:
\begin{itemize}
    \item \textbf{InterpretML}: for EBM implementation \citep{nori2019interpretml}.
    \item \textbf{scikit-survival}: for Cox models, as well as for calculating time-dependent AUC and brier metrics \citep{polsterl2020scikit}.
    \item \textbf{pysurvival}: for Random Survival Forest implementation \citep{pysurvival_cite}.
    \item \textbf{scikit-learn}: for logistic regression as well as LASSO for feature selection \citep{pedregosa2011scikit}.
    \item \textbf{XGBoost}: for running XGBoost \citep{chen2016xgboost}.
\end{itemize}
Additionally, we use the following hyperparameters for our models:
\begin{itemize}
    \item \textbf{EBM}: We use \verb|outer_bags=25| and \verb|inner_bags=10| to ensure robustness, and further use \verb|max_bins=64, interactions=20|, \verb|max_rounds=5000| for expressivity.
    \item \textbf{XGBoost}: We use default hyperparameters.
    \item \textbf{LogReg}: The \verb|saga| solver is used, motivated by the high-dimensionality of our data.
    \item \textbf{CoxPH}: we use the default hyperparameters from scikit-survival.
    \item \textbf{RSF}: we use the default hyperparameters in pysurvival: \verb|max_features=sqrt|, \verb|min_node_size=10|, \verb|sample_size_pct=0.63|.
\end{itemize}

% \clearpage
\section{Feature Selection}\label{sec:feature-selection}

\begin{figure}[t]
    \centering
    \includegraphics[width=0.8\linewidth]{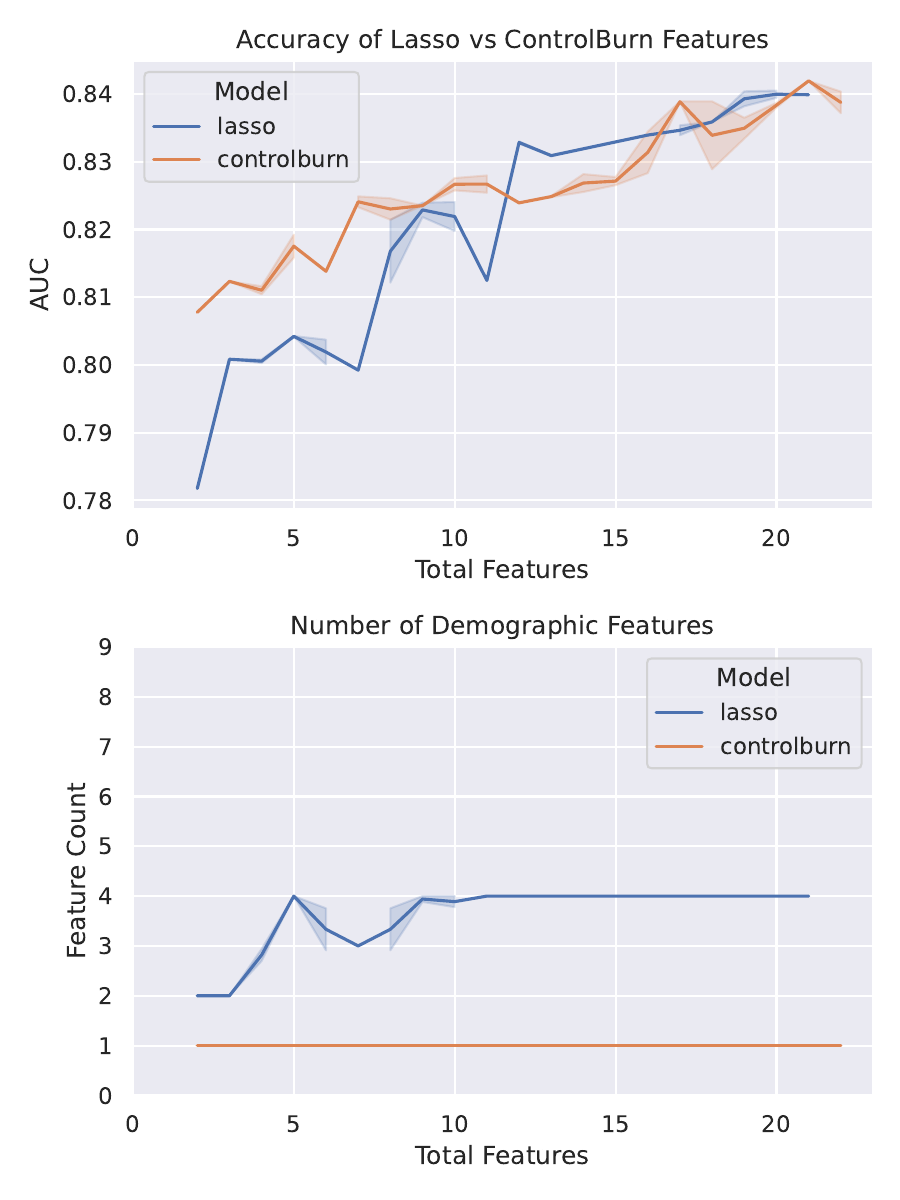}
    \caption{Comparison of linear LASSO and ControlBurn for feature selection. Each model is run for 30 regularization parameters, with each run consisting of 5 trials with different random seeds.
    Each trial is evaluated by running an XGBoost model on the selected features. For categorical features, if any of the one hot encoded features are selected, all are included and are collectively counted as 1 feature. The left plot shows the AUC for 5-year risk prediction for each model. The right plot shows what types of features are being selected. Errors represent standard error across trials if multiple trials result in the same number of features.}
    \label{fig:num_features-to-AUC}
\end{figure}

We assess the ability of ControlBurn (discussed in Section~\ref{sec:ControlBurn}) and the linear LASSO \citep{tibshirani1996regression} to select useful features for predicting heart failure. 
Since the survival stacked data contains not only many samples but also many features, we run feature selection on a 5-year heart failure classification task as a proxy for the full survival analysis problem. 
We compare the performance of ControlBurn and LASSO by evaluating the performance of an XGBoost model trained using the selected features. The results as a function of the number of features selected are shown in Figure \ref{fig:num_features-to-AUC}.

The top plot in Figure~\ref{fig:num_features-to-AUC} demonstrates that ControlBurn selects better features than LASSO in terms of the resulting 5-year risk predictions when the number of features is roughly less than 10.
When the number of features increases past 10, the methods perform similarly.
The bottom plot in Figure~\ref{fig:num_features-to-AUC} helps explain the difference between ControlBurn and LASSO in terms of what types of features are being selected. 
Evidently, ControlBurn selects measurement features more often, while LASSO first selects demographic features before selecting more measurement features.
Since ControlBurn performs comparably to LASSO overall, and better for smaller feature sets, this suggests that demographic features such as gender and race, which the LASSO model selects, may provide no more signal than additional measurement features.
One possible explanation for this behavior is that the demographic features might provide more linear signal and are thus selected by LASSO.
On the other hand, the measurement features often appear to contain comparatively stronger nonlinear signal, see e.g.~Figure~\ref{fig:EBM-ft-imp}, which ControlBurn tends to capture.
These results demonstrate the utility of ControlBurn for selecting an appropriate subset of features for nonlinear prediction models, especially when generating small feature sets.

\section{Additional Experiment Results}

Along with the shape functions and interaction terms shown in Figures \ref{fig:shapefuncs} and \ref{fig:shapefuncs-interactions}, we present the most important features by global feature importance in Figure \ref{fig:EBM-ft-imp}. 
Age appears to be the most important feature, while the feature importance plot also recognises other traditional risk factors as discussed in depth in Section~\ref{sec:discussion}. 
% See Section~\ref{} for a more in-depth analysis of various important factors and their specific shape functions.

\begin{figure*}[h!]
    \centering
        \includegraphics[width=\linewidth]{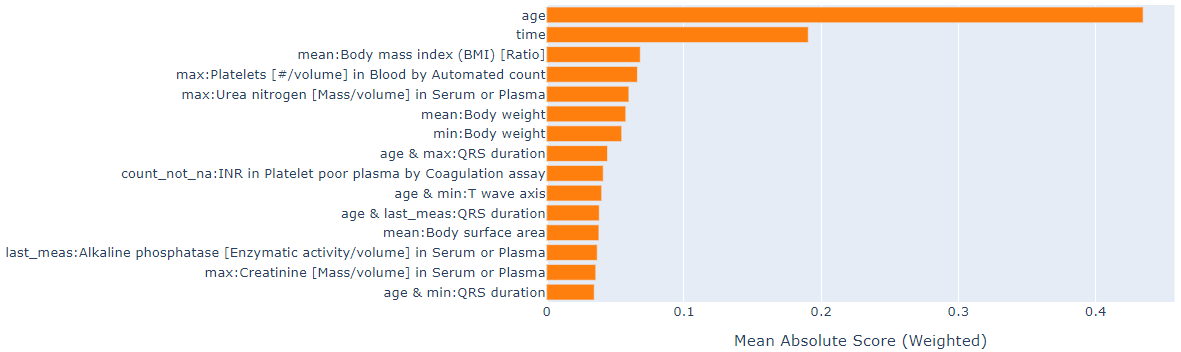}
        \caption{The EBM's feature importances, measured through averaging the feature's absolute contributions to risk over all samples. Age is a significantly better predictor than the other features. Additionally, most interaction terms include age, suggesting age is also an important feature in combination with other features after accounting for its individual contribution.}
    \label{fig:EBM-ft-imp}
\end{figure*}

% \clearpage
\onecolumn
\begin{center}
\fontsize{7}{9}\selectfont
% \LTcapwidth=1\linewidth
\begin{longtable}{lrrrrr}
\caption{This table shows the $10^{\textnormal{th}}, 25^{\textnormal{th}}, 50^{\textnormal{th}}, 75^{\textnormal{th}}$, and $90^{\textnormal{th}}$ percentiles of all features picked up by at least one feature selection method. } \\
\toprule
{\textbf{Percentile}} &     \textbf{10}$^{\textnormal{th}}$ &     \textbf{25}$^{\textnormal{th}}$ &     \textbf{50}$^{\textnormal{th}}$ &     \textbf{75}$^{\textnormal{th}}$ &     \textbf{90}$^{\textnormal{th}}$ \\
\midrule
last\_meas:Pulse rate                                                                                                                   &    60.00 &    66.00 &    75.00 &    84.00 &    94.00 \\
mean:Prothrombin time (PT)                                                                                                             &    11.67 &    12.60 &    13.60 &    14.70 &    17.08 \\
min:Venous oxygen saturation measurement                                                                                               &    45.00 &    63.00 &    73.00 &    84.10 &    93.20 \\
count\_not\_na:Aspartate aminotransferase measurement                                                                                    &     1.00 &     1.00 &     1.00 &     1.00 &     3.00 \\
mean:Urea nitrogen [Mass/volume] in Serum or Plasma                                                                                    &     8.50 &    11.00 &    14.00 &    17.00 &    22.00 \\
min:Diastolic blood pressure                                                                                                           &    57.00 &    63.00 &    70.00 &    78.00 &    84.00 \\
mean:Thyroid stimulating hormone measurement                                                                                           &     0.72 &     1.10 &     1.65 &     2.43 &     3.42 \\
cond:Gastroesophageal reflux disease                                                                                                   &     0.00 &     0.00 &     0.00 &     0.00 &     0.00 \\
max:Cholesterol.total/Cholesterol in HDL [Mass Ratio] in Serum or Plasma                                                               &     3.20 &    51.00 &   115.00 &   150.00 &   183.00 \\
mean:Oxygen [Partial pressure] in Arterial blood                                                                                       &    82.12 &   119.05 &   177.92 &   231.00 &   283.36 \\
mean:Erythrocyte distribution width [Ratio] by Automated count                                                                         &    12.50 &    13.00 &    13.60 &    14.50 &    16.00 \\
mean:Chloride [Moles/volume] in Serum or Plasma                                                                                        &    99.50 &   101.43 &   103.00 &   105.00 &   107.00 \\
max:Triglyceride [Moles/volume] in Serum or Plasma                                                                                     &    50.00 &    70.00 &   104.00 &   160.00 &   241.00 \\
max:Glomerular filtration rate/1.73 sq M.predicted  &  \multirow{2}{*}{57.00} &    \multirow{2}{*}{60.00} &    \multirow{2}{*}{66.00} &    \multirow{2}{*}{102.00} &   \multirow{2}{*}{120.00} \\
\hspace{.5cm} in Serum, Plasma or Blood by Creatinine-based formula (MDRD) &     &    &   &  &   \\
count\_not\_na:Chloride [Moles/volume] in Serum or Plasma                                                                                &     1.00 &     1.00 &     2.00 &     4.00 &    11.00 \\
max:Base excess measurement                                                                                                            &     0.30 &     0.80 &     1.90 &     3.70 &     6.20 \\
min:T wave axis                                                                                                                        &   -10.00 &    10.00 &    30.00 &    48.00 &    68.00 \\
last\_meas:Triglyceride [Moles/volume] in Serum or Plasma                                                                               &    48.00 &    66.00 &    97.00 &   146.00 &   215.00 \\
count\_not\_na:Thyroid stimulating hormone measurement                                                                                   &     1.00 &     1.00 &     1.00 &     1.00 &     1.00 \\
mean:High density lipoprotein measurement                                                                                              &    36.00 &    43.00 &    53.00 &    66.00 &    79.00 \\
mean:Basophils/100 leukocytes in Blood by Automated count                                          &     0.10 &     0.27 &     0.40 &     0.60 &     0.90 \\
last\_meas:Creatine kinase.MB [Mass/volume] in Blood                                                                                    &     0.50 &     0.71 &     1.58 &     3.19 &    10.50 \\
last\_meas:Glucose [Mass/volume] in Serum or Plasma                                                                                     &    84.00 &    91.00 &   100.00 &   116.00 &   145.00 \\
mean:Hemoglobin A1c/Hemoglobin.total in Blood                                                                                          &     5.00 &     5.30 &     5.70 &     6.30 &     7.67 \\
last\_meas:Q-T interval corrected                                                                                                       &   391.00 &   404.00 &   419.00 &   436.00 &   455.00 \\
min:Prothrombin time (PT)                                                                                                              &    11.30 &    12.20 &    13.10 &    13.90 &    14.90 \\
mean:Carbon dioxide, total [Moles/volume] in Serum or Plasma                                                                           &    23.25 &    25.00 &    27.00 &    28.50 &    30.00 \\
min:Body weight                                                                                                                        &  1888.00 &  2160.00 &  2560.00 &  3040.00 &  3555.20 \\
last\_meas:Left ventricular Ejection fraction                                                                                           &    53.20 &    57.70 &    62.10 &    66.30 &    70.20 \\
min:Body surface area                                                                                                                  &     1.54 &     1.67 &     1.84 &     2.04 &     2.22 \\
max:Left ventricular Ejection fraction                                                                                                 &    54.20 &    58.00 &    62.70 &    67.00 &    70.70 \\
max:Platelets [\#/volume] in Blood by Automated count                                                                                   &   173.00 &   208.00 &   252.00 &   309.00 &   385.00 \\
min:Urea nitrogen [Mass/volume] in Serum or Plasma                                                                                     &     6.00 &     9.00 &    12.00 &    16.00 &    20.00 \\
mean:Creatinine [Mass/volume] in Serum or Plasma                                                                                       &     0.65 &     0.76 &     0.90 &     1.07 &     1.26 \\
std:P wave axis                                                                                                                        &     2.12 &     4.60 &     9.37 &    17.79 &    31.11 \\
count\_not\_na:Venous oxygen saturation measurement                                                                                      &     1.00 &     1.00 &     1.00 &     2.00 &     5.00 \\
std:Sodium [Moles/volume] in Serum or Plasma                                                                                           &     0.58 &     1.15 &     2.00 &     2.83 &     3.67 \\
min:Erythrocyte distribution width [Ratio] by Automated count                                                                          &    12.30 &    12.70 &    13.30 &    14.00 &    15.00 \\
min:Chloride measurement, blood                                                                                                        &   101.00 &   105.00 &   108.00 &   111.00 &   113.00 \\
max:Cholesterol in LDL [Mass/volume] in Serum or Plasma by Direct assay                                                                &    72.00 &    91.00 &   114.00 &   139.00 &   165.00 \\
max:Glucose [Mass/volume] in Serum or Plasma                                                                                           &    87.00 &    95.00 &   110.00 &   146.00 &   200.00 \\
last\_meas:Lymphocytes/100 leukocytes in Blood by Automated count                                   &     9.10 &    15.70 &    24.10 &    31.80 &    38.30 \\
std:Q-T interval corrected                                                                                                             &     0.00 &     0.00 &     0.00 &     6.59 &    16.74 \\
mean:P-R Interval                                                                                                                      &   136.00 &   148.00 &   162.00 &   180.00 &   198.00 \\
count\_not\_na:INR in Platelet poor plasma by Coagulation assay                                                                          &     1.00 &     1.00 &     2.00 &     4.00 &     9.00 \\
mean:Systolic blood pressure                                                                                                           &   106.00 &   113.50 &   123.00 &   133.33 &   144.33 \\
mean:Body height                                                                                                                       &    61.00 &    63.00 &    66.00 &    69.00 &    72.00 \\
last\_meas:Venous oxygen saturation measurement                                                                                         &    51.00 &    67.22 &    76.00 &    86.40 &    94.30 \\
min:Cholesterol [Mass/volume] in Serum or Plasma                                                                                       &   130.00 &   152.00 &   178.00 &   204.00 &   231.00 \\
max:QRS duration                                                                                                                       &    78.00 &    84.00 &    90.00 &   100.00 &   112.00 \\
std:Glucose [Mass/volume] in Serum or Plasma                                                                                           &     2.12 &     7.07 &    16.26 &    28.20 &    46.05 \\
drug:aspirin 325 MG Oral Tablet                                                                                                        &     0.00 &     0.00 &     0.00 &     0.00 &     0.00 \\
mean:Globulin [Mass/volume] in Serum                                                                                                   &     2.50 &     2.90 &     3.40 &     3.80 &     4.20 \\
mean:Glucose [Mass/volume] in Serum or Plasma                                                                                          &    86.00 &    93.00 &   103.50 &   121.00 &   147.00 \\
std:Pulse rate                                                                                                                         &     2.83 &     5.29 &     8.12 &    11.31 &    14.96 \\
max:Globulin [Mass/volume] in Serum                                                                                                    &     2.60 &     3.00 &     3.50 &     4.00 &     4.50 \\
max:Hemoglobin [Mass/volume] in Blood                                                                                                  &    11.60 &    12.70 &    13.70 &    14.80 &    15.70 \\
last\_meas:Cholesterol non HDL [Mass/volume] in Serum or Plasma                                                                         &    79.00 &    96.00 &   121.00 &   149.00 &   177.00 \\
min:Oxygen saturation measurement                                                                                                      &    28.00 &    42.00 &    61.00 &    79.00 &    91.00 \\
max:Lactate [Mass/volume] in Blood                                                                                                     &     0.69 &     0.96 &     1.42 &     2.18 &     3.28 \\
min:R-R interval by EKG                                                                                                                &   556.00 &   667.00 &   800.00 &   923.00 &  1053.00 \\
last\_meas:Thyroid stimulating hormone measurement                                                                                      &     0.71 &     1.10 &     1.64 &     2.41 &     3.40 \\
min:Vancomycin [Mass/volume] in Serum or Plasma --trough                                                                               &     3.10 &     5.40 &     8.30 &    12.00 &    16.48 \\
std:Oxygen [Partial pressure] in Venous blood                                                                                          &     0.21 &     2.48 &     7.07 &    14.62 &    28.24 \\
max:Creatinine measurement                                                                                                             &     0.68 &     0.80 &     1.00 &     2.98 &  2100.00 \\
max:Monocytes [\#/volume] in Blood by Manual count                                                                                      &     0.20 &     0.40 &     0.70 &     1.30 &     2.17 \\
max:Body height                                                                                                                        &    61.22 &    63.00 &    66.00 &    69.02 &    72.00 \\
last\_meas:INR in Platelet poor plasma by Coagulation assay                                                                             &     1.00 &     1.00 &     1.10 &     1.20 &     1.40 \\
cond:Atrial fibrillation                                                                                                               &     0.00 &     0.00 &     0.00 &     0.00 &     0.00 \\
mean:Anion gap in Serum or Plasma                                                                                                      &     5.00 &     6.22 &     7.73 &     9.00 &    11.00 \\
count\_not\_na:Low density lipoprotein measurement                                                                                       &     1.00 &     1.00 &     1.00 &     1.00 &     1.00 \\
mean:Diastolic blood pressure                                                                                                          &    63.50 &    69.00 &    75.00 &    81.00 &    87.00 \\
last\_meas:Prothrombin time (PT)                                                                                                        &    11.60 &    12.50 &    13.50 &    14.50 &    16.70 \\
last\_meas:Creatinine measurement                                                                                                       &     0.67 &     0.80 &     1.00 &     1.92 &  1980.00 \\
max:Cholesterol [Mass/volume] in Serum or Plasma                                                                                       &   138.00 &   160.00 &   186.00 &   215.00 &   244.00 \\
std:Urea nitrogen [Mass/volume] in Serum or Plasma                                                                                     &     0.58 &     1.41 &     2.79 &     4.24 &     6.45 \\
min:Creatinine [Mass/volume] in Serum or Plasma                                                                                        &     0.60 &     0.70 &     0.82 &     1.00 &     1.20 \\
mean:Body weight                                                                                                                       &  1929.47 &  2208.00 &  2616.00 &  3100.55 &  3625.07 \\
min:Creatinine measurement                                                                                                             &     0.66 &     0.80 &     1.00 &     1.85 &  1904.90 \\
max:Urea nitrogen [Mass/volume] in Serum or Plasma                                                                                     &     9.00 &    12.00 &    15.00 &    20.00 &    27.00 \\
min:Creatine kinase [Enzymatic activity/volume] in Serum or Plasma                                                                     &    31.00 &    54.00 &    93.00 &   183.00 &   477.00 \\
count\_not\_na:Glucose measurement, blood                                                                                                &     1.00 &     1.00 &     1.00 &     3.00 &     7.00 \\
max:T wave axis                                                                                                                        &     4.00 &    20.00 &    39.00 &    58.00 &    86.00 \\
mean:Calcium [Mass/volume] in Serum or Plasma                                                                                          &     8.28 &     8.65 &     9.00 &     9.35 &     9.60 \\
last\_meas:Leukocytes [\#/volume] in Specimen by Automated count                                                                         &     4.50 &     5.60 &     7.00 &     9.00 &    11.60 \\
mean:Venous oxygen saturation measurement                                                                                              &    54.15 &    68.00 &    76.42 &    85.60 &    93.50 \\
min:Body height                                                                                                                        &    61.00 &    63.00 &    66.00 &    69.00 &    72.00 \\
drug:hydrochlorothiazide 25 MG Oral Tablet                                                                                             &     0.00 &     0.00 &     0.00 &     0.00 &     0.00 \\
min:Pulse rate                                                                                                                         &    55.00 &    61.00 &    69.00 &    78.00 &    88.00 \\
std:QRS axis                                                                                                                           &     2.08 &     4.24 &     8.33 &    14.59 &    24.75 \\
min:Left ventricular Ejection fraction                                                                                                 &    52.24 &    57.10 &    61.70 &    66.00 &    70.00 \\
count\_not\_na:Calcium [Mass/volume] in Serum or Plasma                                                                                  &     1.00 &     1.00 &     2.00 &     4.00 &    11.00 \\
max:Systolic blood pressure                                                                                                            &   109.00 &   118.00 &   129.00 &   142.00 &   156.00 \\
mean:Measurement of venous partial pressure of carbon dioxide                                                                          &    34.05 &    38.60 &    43.10 &    47.40 &    51.50 \\
min:Hemoglobin level estimation                                                                                                        &    10.00 &    11.80 &    13.10 &    14.30 &    15.30 \\
last\_meas:QRS duration                                                                                                                 &    76.00 &    82.00 &    90.00 &    98.00 &   110.00 \\
count\_not\_na:Urea nitrogen [Mass/volume] in Serum or Plasma                                                                            &     1.00 &     1.00 &     2.00 &     4.00 &    11.00 \\
min:Leukocytes [\#/volume] in Blood by Automated count                                                                                  &     3.78 &     4.35 &     4.97 &     7.30 &    10.60 \\

last\_meas:Alkaline phosphatase [Enzymatic activity/volume]   &  \multirow{2}{*}{50.00} &    \multirow{2}{*}{61.00} &    \multirow{2}{*}{78.00} &    \multirow{2}{*}{99.00} &   \multirow{2}{*}{127.00} \\
\hspace{.5cm} in Serum or Plasma &     &    &   &  &   \\
min:P wave axis                                                                                                                        &     0.00 &    23.00 &    43.00 &    59.00 &    70.00 \\
last\_meas:Platelets [\#/volume] in Blood by Automated count                                                                             &   153.00 &   191.00 &   233.00 &   281.00 &   338.00 \\
mean:Leukocytes [\#/volume] in Blood by Automated count                                                                                 &     4.00 &     4.46 &     5.12 &     8.10 &    11.50 \\
max:Monocytes/100 leukocytes in Blood by Automated count                                           &     5.10 &     6.50 &     8.10 &    10.40 &    13.40 \\
max:Thyroid stimulating hormone measurement                                                                                            &     0.73 &     1.12 &     1.67 &     2.46 &     3.50 \\
max:Lymphocytes [\#/volume] in Blood by Automated count                                                                                 &     1.00 &     1.39 &     1.85 &     2.42 &     3.22 \\
min:Alkaline phosphatase [Enzymatic activity/volume] in Serum or Plasma                                                                &    46.00 &    57.00 &    72.00 &    90.00 &   113.00 \\
last\_meas:Erythrocyte distribution width [Ratio] by Automated count                                                                    &    12.50 &    13.00 &    13.60 &    14.40 &    15.90 \\
mean:Body mass index (BMI) [Ratio]                                                                                                     &    20.76 &    23.02 &    26.11 &    30.12 &    35.01 \\
count\_not\_na:P wave axis                                                                                                               &     1.00 &     1.00 &     1.00 &     2.00 &     4.00 \\
mean:Pulse rate                                                                                                                        &    61.00 &    67.80 &    75.00 &    83.33 &    91.80 \\
count\_not\_na:Hemoglobin level estimation                                                                                               &     1.00 &     1.00 &     1.00 &     1.00 &     4.00 \\
mean:Fractional oxyhemoglobin in Arterial blood                                                                                        &    94.30 &    96.25 &    97.30 &    97.90 &    98.40 \\
mean:Lymphocytes/100 leukocytes in Blood by Automated count                                        &     9.78 &    15.40 &    23.20 &    30.90 &    37.30 \\
age                                                                                                                                    &    26.00 &    34.00 &    48.00 &    62.00 &    73.00 \\
max:Body surface area                                                                                                                  &     1.57 &     1.70 &     1.88 &     2.08 &     2.27 \\
mean:QRS duration                                                                                                                      &    77.89 &    82.67 &    90.00 &    98.00 &   109.00 \\
mean:T wave axis                                                                                                                       &     1.00 &    16.00 &    35.00 &    52.00 &    73.50 \\
mean:MCHC [Mass/volume] by Automated count                                                                                             &    32.70 &    33.30 &    33.90 &    34.40 &    34.80 \\
min:INR in Platelet poor plasma by Coagulation assay                                                                                   &     1.00 &     1.00 &     1.10 &     1.10 &     1.20 \\
std:Bilirubin.direct [Mass/volume] in Serum or Plasma                                                                                  &     0.00 &     0.00 &     0.05 &     0.14 &     0.71 \\
mean:Cholesterol [Mass/volume] in Serum or Plasma                                                                                      &   135.93 &   157.00 &   182.00 &   209.00 &   235.00 \\
last\_meas:Body weight                                                                                                                  &  1923.20 &  2208.00 &  2620.80 &  3104.00 &  3632.00 \\
mean:Body surface area                                                                                                                 &     1.56 &     1.68 &     1.86 &     2.06 &     2.25 \\
max:MCHC [Mass/volume] by Automated count                                                                                              &    32.90 &    33.60 &    34.20 &    34.80 &    35.30 \\
min:Sodium [Moles/volume] in Serum or Plasma                                                                                           &   132.00 &   135.00 &   137.00 &   139.00 &   141.00 \\
min:Aspartate aminotransferase [Enzymatic activity/volume] in Serum or Plasma                                                          &    13.00 &    17.00 &    21.00 &    27.00 &    37.00 \\
last\_meas:Systolic blood pressure                                                                                                      &   104.00 &   112.00 &   122.00 &   134.00 &   147.00 \\
max:Q-T interval corrected                                                                                                             &   395.00 &   408.00 &   423.00 &   442.00 &   463.20 \\
min:MCHC [Mass/volume] by Automated count                                                                                              &    32.20 &    32.90 &    33.60 &    34.20 &    34.60 \\
max:Basophils [\#/volume] in Blood by Automated count                                                                                   &     0.01 &     0.02 &     0.04 &     0.06 &     0.10 \\
last\_meas:QRS axis                                                                                                                     &   -28.00 &    -2.00 &    27.00 &    55.00 &    74.00 \\
max:Myoglobin [Presence] in Serum or Plasma                                                                                            &    25.00 &    32.00 &    46.00 &    82.00 &   194.00 \\
last\_meas:Anion gap in Serum or Plasma                                                                                                 &     5.00 &     6.00 &     8.00 &     9.00 &    11.00 \\
mean:Chloride measurement, blood                                                                                                       &   104.00 &   106.73 &   109.75 &   112.33 &   114.67 \\
last\_meas:Body temperature                                                                                                             &    97.20 &    97.60 &    98.00 &    98.40 &    98.70 \\
std:Prothrombin time (PT)                                                                                                              &     0.14 &     0.41 &     0.80 &     1.68 &     4.62 \\
max:Thyroxine (T4) free [Mass/volume] in Serum or Plasma                                                                               &     0.90 &     1.00 &     1.20 &     1.37 &     1.60 \\
last\_meas:T wave axis                                                                                                                  &    -1.00 &    16.00 &    35.00 &    53.00 &    74.00 \\
std:Cholesterol.total/Cholesterol in HDL [Mass Ratio] in Serum or Plasma                                                               &    43.99 &    61.17 &    79.08 &    97.72 &   116.04 \\
min:QRS duration                                                                                                                       &    76.00 &    82.00 &    88.00 &    96.00 &   106.00 \\
min:Chloride [Moles/volume] in Serum or Plasma                                                                                         &    97.00 &   100.00 &   102.00 &   104.00 &   106.00 \\
mean:Creatinine measurement                                                                                                            &     0.67 &     0.80 &     1.00 &     2.80 &  1980.00 \\
min:Thyrotropin [Units/volume] in Serum or Plasma                                                                                      &     0.48 &     0.89 &     1.39 &     2.10 &     3.06 \\
drug:1 ML hydromorphone hydrochloride 2 MG/ML Injection                                                                                &     0.00 &     0.00 &     0.00 &     0.00 &     0.00 \\
max:Body weight                                                                                                                        &  1964.74 &  2243.40 &  2663.16 &  3160.51 &  3700.20 \\
mean:Q-T interval corrected                                                                                                            &   392.00 &   405.00 &   419.50 &   436.00 &   453.50 \\
min:aPTT in Platelet poor plasma by Coagulation assay                                                                                  &    25.10 &    27.20 &    29.80 &    32.80 &    36.80 \\
max:Creatinine [Mass/volume] in Serum or Plasma                                                                                        &     0.70 &     0.80 &     0.97 &     1.13 &     1.40 \\
mean:Glomerular filtration rate/1.73 sq M.predicted  &  \multirow{2}{*}{55.00} &    \multirow{2}{*}{60.00} &    \multirow{2}{*}{60.00} &    \multirow{2}{*}{60.00} &   \multirow{2}{*}{60.00} \\
\hspace{.5cm}[Volume Rate/Area] in Serum, Plasma or Blood by CKD-EPI &     &    &   &  &   \\
min:Hematocrit [Volume Fraction] of Blood                                                                                              &    30.00 &    35.00 &    38.70 &    42.00 &    45.00 \\
last\_meas:Globulin [Mass/volume] in Serum                                                                                              &     2.50 &     2.90 &     3.40 &     3.80 &     4.20 \\
min:Body temperature                                                                                                                   &    96.80 &    97.30 &    97.70 &    98.10 &    98.50 \\
last\_meas:Creatinine [Mass/volume] in Serum or Plasma                                                                                  &     0.64 &     0.75 &     0.90 &     1.10 &     1.30 \\
min:Eosinophils [\#/volume] in Blood by Automated count                                                                                 &     0.00 &     0.02 &     0.08 &     0.16 &     0.30 \\
max:Body mass index (BMI) [Ratio]                                                                                                      &    21.10 &    23.44 &    26.63 &    30.78 &    35.87 \\
drug:fentanyl 0.05 MG/ML Injection                                                                 &     0.00 &     0.00 &     0.00 &     0.00 &     1.00 \\
drug:2 ML ondansetron 2 MG/ML Injection                                                                                                &     0.00 &     0.00 &     0.00 &     0.00 &     1.00 \\
last\_meas:Urea nitrogen [Mass/volume] in Serum or Plasma                                                                               &     8.00 &    11.00 &    14.00 &    17.00 &    22.00 \\
min:Eosinophils/100 leukocytes in Blood by Automated count                                         &     0.00 &     0.20 &     1.00 &     2.10 &     3.70 \\
\bottomrule
\label{tab:cohort_stats_all_fts}
\end{longtable}
% \end{table}
\end{center}
% \clearpage
% \twocolumn

% \begin{figure}[h!]
%     \centering
%     \includegraphics[width=\linewidth]{figures/survival_time_plots.pdf}
%     \caption{Distribution of survival time for patients in our cohort, for all patients, censored patients, and uncensored (heart failure) patients.}
%     \label{fig:time_hists}
% \end{figure}

% \begin{figure}[h!]
%     \centering
%     \includegraphics[width=\linewidth]{figures/cohort_feature_plots.pdf}
%     \caption{Plots showing the relationship between heart failure and the demographic features in the cohort.}
%     \label{fig:cohort_feature_plots}
% \end{figure}

\end{document}